\documentclass[12pt]{amsart}
\usepackage{newtxtext, newtxmath}
\usepackage[english]{babel}
\usepackage[margin=1in]{geometry}
\usepackage{xcolor}

\usepackage{pdflscape}
\usepackage{graphicx}
\usepackage{schemabloc}
\usepackage{framed} 
\usepackage{pifont} 
\usepackage{rotating}
\usepackage[normalem]{ulem}

\usepackage{hyperref}
\hypersetup{
	colorlinks   = true,
	citecolor    = gray
}
\hypersetup{linkcolor=blue}

\usepackage{amsmath}
\usepackage{amsthm}
\usepackage{amssymb}
\usepackage{stackengine}
\usepackage[all,cmtip]{xy}
\usepackage{mathtools} 
\usepackage{extarrows} 
\usetikzlibrary{arrows.meta}

\newcommand{\catdigraph}{\mathcal{G}}
\newcommand{\catgraph}{\catdigraph}
\newcommand{\mor}{\operatorname{Mor}}
\newcommand{\N}{\mathbb{N}}

\newcommand{\funct}{\operatorname{Funct}}

\newcommand{\seta}[1][]{\xrightarrow[\vphantom{=}\smash{\raisebox{0.5ex}{=}}]{#1}}

\newcommand{\bx}{\mathbf{X}}
\newcommand{\by}{\mathbf{Y}}

\newcommand{\ff}{\mathfrak{F}}
\newcommand{\ffo}{\ff^\Omega}
\newcommand{\oa}{\Omega^\ast}
\newcommand{\ffoa}{\ff^{\Omega^\ast}}
\newcommand{\ffh}{\widehat{\ff}}

\newcommand{\net}{\mathcal{N}}
\newcommand{\netex}{\mathcal{N}_\mathrm{ext}}
\newcommand{\ultranet}{\mathcal{U}}
\newcommand{\ultranetex}{\mathcal{U}_\mathrm{ext}}
\newcommand{\disnet}{\mathcal{N}_\mathrm{dis}}
\newcommand{\dis}{\operatorname{dis}}

\newcommand{\subpart}{\operatorname{SubPart}}
\newcommand{\R}{\mathbb{R}}

\theoremstyle{definition}
\newtheorem{theorem}{Theorem}[section]
\newtheorem{definition}[theorem]{Definition}
\newtheorem{proposition}[theorem]{Proposition}
\newtheorem{remark}[theorem]{Remark}

\newtheorem{example}[theorem]{Example}
\newtheorem{lemma}[theorem]{Lemma}
\newtheorem{corollary}[theorem]{Corollary}

\theoremstyle{remark}
\newtheorem{claim}{Claim}[theorem]
\newenvironment{claimproof}{\noindent\textit{Proof.}}{\hfill $\blacksquare$\par\medskip}

\begin{document} 

\title{Motivic clustering schemes for directed graphs}

\author{Facundo M\'emoli}
\address{Departments of Mathematics and Computer Science and Engineering.  the Ohio State University. USA.}
\email{memoli@math.osu.edu}
\thanks{FM was  partially supported by NSF grants DMS-1723003, CCF-1740761, and AF-1526513. GV was partly supported by the Coordena{\c c}{\~a}o de Aperfei{\c c}oamento de Pessoal de N{\'i}vel Superior -- Brasil (CAPES) -- Finance Code 001}

\author{Guilherme Vituri F. Pinto}
\address{Departamento de Matem\'atica. UNESP, Rio Claro, Brazil.}
\email{vituri.vituri@gmail.com}
\date{\today}

\begin{abstract}
Motivated by the concept of network motifs we construct certain clustering methods (functors) which are parametrized by a given collection of motifs (or representers). 
\end{abstract}

\maketitle
\tableofcontents

\section{Introduction}
Clustering is an useful procedure to identify subgroups that exhibit some kind of proximity or similarity in datasets. The practice of clustering metric spaces, with its many algorithms, is well developed~\cite{RuiWunsch05}. 
From a theoretical perspective, in~\cite{Carlsson10ClCl, Carlsson10ChSt} the authors invoked the concept of ``functoriality'' to define desirable properties of maps that send a metric space to a hierarchical clustering of its vertices, and proved that a unique method (i.e. functor), \textit{single linkage} hierarchical clustering, satisfies these properties. When restricting to the category of metric spaces and \textit{injective} maps, the same authors found an infinite family of standard (or flat) and hierarchical clustering methods satisfying these same properties.
This is a counterpart to a result of Kleinberg~\cite{kleinberg}, which states that there is no method of standard (as opposed to hierarchical) clustering satisfying certain natural conditions. 

However, when a dataset can no longer be represented as a metric space, the interpretation of a clustering can be more difficult. In~\cite{Carlssonetal13AxiCons}, the authors extended their previous work and studied hierarchical clustering of \textit{dissimilarity networks}: pairs $(X, A)$ where $A\colon X \times X \to \mathbb{R}^+$ satisfies $A(x, x') = 0$ $\Leftrightarrow$ $x = x'$. Under reasonable conditions, the methods of hierarchical clustering they identified were well behaved and many results, such as stability with respect to a suitable notion of distance between networks, were proved.

We further generalize this line of work and study pairs $(X, w_X)$, where $X$ is a finite set and $w_X\colon X \times X \to \mathbb{R} \cup \{+\infty\}$ is any function. 
Any such object, called an \textit{extended network} (and whose category is denoted by $\netex$) can be regarded as a filtration of (directed) graphs. 
By studying endofunctors $\ff\colon \catdigraph \to \catdigraph$ (where $\catdigraph$ is the category of graphs) we are able to create many different clustering functors on $\catdigraph$. 
These endofunctors naturally give rise to endofunctors on $\netex$ whose output is a generalization of an ultrametric space (or, equivalently, of a dendrogram). 
This approach, although not as general as directly dealing with general endofunctors on $\netex$, turns out to be very useful and simplifies many proofs. 
For example, the notion of quasi-clustering from~\cite{Carlssonetal13AxiCons} can be obtained as in Definition~\ref{def:: quotient graph}, whenever $G$ is a transitive graph. 
Besides this, the study of endofunctors on $\catdigraph$ is interesting on its own.

We borrow the notion of representable methods from~\cite{Carlsson2016ExcisiveHC} and adapt it to the context of endofunctors on $\catdigraph$. 
Given a set of graphs $\Omega$ (the \textit{representers} or \textit{motifs}), we define $\ffo\colon \catdigraph \to \catdigraph$ as a functor that captures ``interesting shapes'' based on $\Omega$.  

From the point of view of applications, in the exploratory stage of the analysis of a given, experimentally measured network, applying a clustering procedure may be useful in identifying different structures within the given network. There is evidence  suggesting that biological networks are assembled of building blocks that do not appear to be random~\cite{holland1974statistical,holland1976local}. On the contrary, in the context of metabolic networks recent work~\cite{newman2003structure} indicates that only certain particular building blocks are the ones that are most often observed. This of course suggests that when looking at a given network, identifying these building blocks, or motifs, and invoking them in order to process the information conveyed by the network, is of chief importance.

\subsection*{Structure of the paper.} Section \ref{sec: background} covers the basic background used throughout the paper. In Section~\ref{sec: endofunctors} we introduce some endofunctors on $\catdigraph$, define a partial order on these and prove some useful properties of such endofunctors; we see that some information about an endofunctor $\ff\colon \catdigraph \to \catdigraph$ can be obtained by just applying it to a small graph with two vertices. 
Section~\ref{sec: representable} deals with representable endofunctors, which are symmetric as a direct consequence of the definition. In that section we show that many of the functors described in previous sections turned out to be representable. 
Section~\ref{sec: pointed representable} presents the notion of \textit{pointed} representable functor, which tries to remove the forced symmetry in the definition of usual representable functors. A composition rule is obtained, and it turns out that every endofunctor is pointed representable. 
In Section~\ref{sec: simplifications} we completely characterize the relation between sets of representers $\Omega_1$ and $\Omega_2$ so that the functors represented by them are equal, and show how to ``simplify'' a given family of representers.  
In Section~\ref{sec: relations to hc} we show how an endofunctor $\ff\colon \catdigraph \to \catdigraph$ induces a functor $\ffh\colon \netex \to \netex$, and explain how properties of $\ff$ such as symmetry and transitivity are ``transferred to'' $\ffh$. Thus, when $\ff$ is symmetric and transitive, the functor $\ffh$ can be seen as a hierarchical clustering method of extended networks. Moreover, if $\ff\colon \catdigraph \to \catdigraph$ is  a non trivial functor, we show a stability result: the distance between $\ffh(\bx)$ and $\ffh(\by)$ is bounded by the distance between $\bx$ and $\by$, for any pair of networks $\bx$ and $\by$. 

\section{Background and notation} \label{sec: background}

A (directed) graph is a pair $G = (V, E)$ where $V$ is a finite set and $E \subseteq V \times V$ is such that $E \supseteq \Delta(V) \coloneqq  \{(v,v), \, v \in V\}$. The elements of $V$ are called \textit{vertices}, the elements of $E$ are called \textit{arrows} (or \textit{edges}), and $|G|$ denotes the cardinality of $V$. Notice that with this definition directed graphs have all self loops. 
We denote an arrow $(v, v') \in E \setminus \Delta(E)$ by $v \xrightarrow{G} v'$ or, when the context is clear, simply by $v \to v'$. Also, $v \nrightarrow v'$ means $(v, v') \notin E$ and $v \leftrightarrow v'$ means both $v \to v'$ and $v' \to v$. We denote the fact that $v \xrightarrow{G} v'$ or $v = v'$ by $v \seta[G] v'$, or, more simply, by $v \seta[] v'$. In all illustrations below we will omit depicting self loops. 

To denote that $v$ is a vertex of $G$, we can write $v \in G$ or $v \in V$.

The category $\catdigraph$ of graphs has as objects all graphs and the morphisms are given by 
\[
\mor_{\catdigraph}(G, G') \coloneqq \big\{\phi \colon V \rightarrow V' \; | \; (\phi \times \phi)(E)\subseteq E' \big\},
\]
that is: $v \xrightarrow{G} v'$ implies $\phi(v) \seta[G'] \phi(v')$, for graphs $G = (V, E), G' = (V', E')$. We call any such map a \textit{graph map}, and denote an element $\phi \in \mor_\catdigraph(G, G')$ by $\phi\colon G \to G'$.

Given a graph map $\phi\colon G \to G'$, whenever we want to emphasize that $v_1', v_2'$ are in $\phi(G)$, we will write $\phi\colon G \to (G', v_1', v_2')$. If, even more, we write $\phi\colon (G, v_1, v_2) \to (G', v_1', v_2')$, this will mean that $\phi(v_1) = v_1'$ and $\phi(v_2) = v_2'$. 

The \emph{disjoint union of $G$ and $G'$}, denoted by $G \sqcup G'$, is the graph with vertex set $V \sqcup V'$ and arrow set $E \sqcup E'$. 

Two graphs $G$ and $G'$ are \textit{isomorphic} if there are graph maps $\phi\colon G \to G'$ and $\phi'\colon G' \to G$ such that $\phi \circ \phi'$ and $\phi' \circ \phi$ are the identity maps on $G'$ and $G$, respectively. Any such $\phi$ is called an \emph{isomorphism} between $G$ and $G'$. Thus, in this case, $G$ is obtained from $G'$ by a relabeling of the vertices. Whenever $G$ and $G'$ are isomorphic we may write $G\cong G'$.

When $G = (V, E)$ and  $G' = (V', E')$ are graphs with $V \subseteq V'$ and the inclusion map $i\colon G \to G'$ given by $i(v) = v$, $\forall v \in V$, is a graph map, we will denote this simply by $G \hookrightarrow G'$. In this setting, denote by $G' \cap V$ the graph $(V, E' \cap (V \times V))$.

Consider some interesting subcategories of $\catdigraph$:
\begin{itemize}
\item $\catdigraph^\mathrm{sym}$, whose objects are \textit{symmetric} graphs (that is: $v \xrightarrow{G} v'$ implies $v' \xrightarrow{G} v$).
\item $\catdigraph^\mathrm{trans}$, whose objects are \textit{transitive} graphs (that is: $v \xrightarrow{G} v'$ and $v' \xrightarrow{G} v''$ implies $v \seta[G] v''$).
\item $\catdigraph^\mathrm{clust} = \catdigraph^\mathrm{sym} \cap \catdigraph^\mathrm{trans}$, whose objects are symmetric and transitive graphs, which later we will regard as encoding a  \textit{clustering} of their sets of vertices.
\end{itemize}

Some standard notions of connectivity on graphs are the ones that follow. A pair of vertices $(v, v')$ on a graph $G=(V,E)$ is called:
\begin{itemize}
\item  \emph{strongly connected} if there is a sequence $v_1, \ldots, v_k \in V$, such that $v = v_1$, $v' = v_k$, and $v_i \to v_{i+1}$ for each $i$. We denote this by $v \leadsto v'$ in $G$. If, moreover, $v_k \xrightarrow{G} v_1$, the sequence $v_1, \ldots, v_k, v_1$ is a \textit{cycle} of size $k$. 
 
\item \emph{weakly connected} if there is a sequence $v_1, \ldots, v_k \in V$, such that $v = v_1$, $v' = v_k$, and $v_i \to v_{i+1}$ or $v_{i+1} \to v_i$, for each $i$. 
\end{itemize}

Given two sets $A, B \subseteq V \times V$, define
\[
A \otimes B \coloneqq \{(v, v') \in V \times V \; | \;
\exists v_1 \in V \mbox{s. t.} (v, v_1) \in A \mbox{ and } (v_1, v') \in B).
\]

Now, given a graph $G = (V, E)$, one defines $E^{(2)} \coloneqq E \otimes E$ and, in general, for $m \in \N$, 
\[
E^{(m+1)}\coloneqq E^{(m)} \otimes E.
\]

If $N$ is such that $E^{(m)} = E^{(N)}$ for all $m > N$, define $E^{(\infty)} \coloneqq E^{(N)}$. Notice that $(v, v') \in E^{(\infty)}$ $\Leftrightarrow$ $v \leadsto v'$.

Here are some important graphs that will appear several times in the text:
\begin{itemize}
\item $K_n$ is the \textit{complete graph} with $n$ vertices $a_1, \ldots, a_n$ and all possible arrows.
\item $D_n$ is the \textit{discrete graph} (or \textit{totally disconnected graph}) with $n$ vertices $a_1, \ldots, a_n$ and no arrows.
\item For a given graph $G = (V, E)$, we will denote by $K_{(G)}$ and $D_{(G)}$ the complete graph and the totally disconnected graph with vertex set $V$, respectively.
\item $L_n$ is the \textit{line graph} with $n$ vertices $a_1, \ldots, a_n$ and arrows $a_i \to a_{i+1}$, $i=1, \ldots, n-1$.
\item $T_n$ is the \textit{transitive line graph} with $n$ vertices $a_1, \ldots, a_n$ and arrows $a_i \to a_j$, for any $1 \leq i < j \leq n$.
\item $C_n$ is the \textit{cycle graph} obtained by adding the arrow $a_n \to a_1$ to $L_n$.
\end{itemize}

For any of the above graphs, its vertices will be called $a_1, \ldots, a_n$ unless stated otherwise.

\section{Endofunctors} \label{sec: endofunctors}
\begin{definition}
A functor $\ff\colon \catdigraph \rightarrow \catdigraph$ is called \textit{vertex preserving} if for any graph $G = (V, E) \in \catdigraph$, the graph $\ff(G)$ has vertex set $V$ and, if given any graph map $\phi\colon G \to G'$, we have $\ff(\phi) = \phi$. We will henceforth denote by $\funct(\catdigraph, \catdigraph)$ the collection of all such functors. \emph{All functors in this work are assumed to be vertex preserving.} Whenever we say that $\ff$ is an endofunctor, we mean that $\ff \in \funct(\catdigraph, \catdigraph)$.

That $\ff$ is a functor means that for every $G, G' \in \catdigraph$ and every graph map $\phi\colon G \to G'$, we have graphs $\ff(G)$ and $\ff(G')$, and the map $\ff(\phi)\colon \ff(G) \to \ff(G')$ is a graph map.

\[
\xymatrix{
G \ar[r]^\phi \ar@{=>}[d]^\ff & G' \ar@{=>}[d]^\ff \\
\ff(G) \ar[r]^{\ff(\phi)} & \ff(G')
}
\] 

We will regard two endofunctors $\ff_1, \ff_2$ as equal when $\ff_1(G) = \ff_2(G)$, for all $G \in \catdigraph$.

We say that $\ff \in \funct(\catdigraph, \catdigraph)$ is \textit{symmetric} if $\ff (\catdigraph) \subseteq \catdigraph^\mathrm{sym}$, and that $\ff$ is \textit{transitive} (resp.\ \textit{clustering}) if $\ff(\catdigraph) \subseteq \catdigraph^\mathrm{trans}$ (resp.\ $\ff(\catdigraph) \subseteq \catdigraph^\mathrm{clust}$). 
\end{definition}

\begin{definition}
Given two endofunctors $\ff_1$ and $\ff_2$, define $\ff_1 \cup \ff_2(G) = (V, E_1 \cup E_2)$, where $G = (V, E)$,  $\ff_1(G) = (V, E_1)$ and $\ff_2(G) = (V, E_2)$.
\end{definition}

\begin{example} Here are some endofunctors that will be used in the sequel:
\begin{itemize}
\item \textbf{Full disconnection}: $\ff^\mathrm{disc}$ taking $G = (V,E)$ to the totally disconnected graph $D_{(G)}$, that is, $\ff(G) = (V, \Delta(V))$.

\item \textbf{Connected component}: $\ff^\mathrm{conn}$, where $v \xrightarrow{\ff^\mathrm{conn}(G)} v'$ if $v$ and $v$ are weakly connected.

\item \textbf{Full completion}: $\ff^\mathrm{comp}$ taking $G = (V,E)$ to complete graph $K_{(G)} = (V, V \times V)$.

\item \textbf{Reversion}: $\ff^\mathrm{rev}$ taking $(V,E)$ to $(V, E^\mathrm{rev})$, where $E^\mathrm{rev} = \{(v', v) \; | \; (v, v') \in E\}.$ 

\item \textbf{Lower symmetrization}: $\ff^\mathrm{ls}$ taking $(V,E)$ to $(V,E\cap E^\mathrm{rev})$.

\item \textbf{Identity}: $\ff^\mathrm{id}$ the identity endofunctor.

\item \textbf{Upper symmetrization}: $\ff^\mathrm{us}$ taking $(V,E)$ to $(V,E \cup E^\mathrm{rev})$.

\item $\mathbf{m}$-\textbf{Power}: for $m \in \N$, $\ff^{[m]}$ taking $(V,E)$ to $(V,E^{(m)})$.

\item \textbf{Transitive closure}: $\ff^\mathrm{tc}$ taking $G = (V,E)$ to $(V,E^{(\infty)})$, that is: $v \xrightarrow{\ff^\mathrm{tc}(G)} v'$ if $v \leadsto v'$ in $G$.

\end{itemize}
\end{example}

\begin{remark}
The ``inversion map'' given by $(V,E) \mapsto (V,E^\mathrm{inv})$, where $E^\mathrm{inv} = \Delta(V) \cup (V \times V \backslash E)$, is not a functor. To see why it fails, just consider the inclusion $D_2 \hookrightarrow K_2$.
\end{remark} 

\begin{definition} \label{def: partial order}
Define the following partial order on $\funct(\catdigraph, \catdigraph)$: $\ff_1 \leq \ff_2$ $\Leftrightarrow$ $\ff_1(G) \hookrightarrow \ff_2(G)$ for all $G \in \catdigraph$.
\end{definition}

\begin{definition} An endofunctor $\ff$ is called \emph{arrow increasing} if for any $G \in \catdigraph$, $v \xrightarrow{G} v'$ implies $v \xrightarrow{\ff(G)} v'$, that is, $G \hookrightarrow \ff(G)$. According to Definition~\ref{def: partial order}, this is equivalent to $\ff^\mathrm{id} \leq \ff$. Analogously, we say that $\ff$ is called \textit{arrow decreasing} if $\ff \leq \ff^\mathrm{id}$. 
\end{definition}

\begin{remark}
It is clear that $\ff^\mathrm{ls} \leq \ff^\mathrm{id} \leq \ff^\mathrm{us}$. Notice that if $\ff$ is arrow increasing, then $\ff(L_2) \in \{L_2, K_2\}$. This condition is also sufficient, as we prove next.
\end{remark}

\begin{proposition}\label{theo:increasing}
Let $\ff$ be an endofunctor. Then, $\ff$ is arrow increasing $\Leftrightarrow$ $\ff(L_2) \in \{L_2, K_2\}$.
\end{proposition}

\begin{proof}
Let $G$ be a graph and suppose $v \xrightarrow{G} v'$. The $\phi\colon (L_2, a_1, a_2) \to (G, v, v')$ is a graph map. By functoriality, $\phi\colon (\ff(L_2), a_1, a_2) \to (\ff(G), v, v')$ is a graph map. If $\ff(L_2) = L_2$ or $K_2$, then $v \xrightarrow{\ff(G)} v'$. Thus, $G \hookrightarrow \ff(G)$.
\end{proof}

\begin{remark} \label{rem: ffl2 and k2}
Even when $\ff(L_2) = K_2$, we cannot ensure that $\ff$ is symmetric. Indeed, let $\ff = \ff^\mathrm{us} \cup \ff^\mathrm{tc}$. Then $\ff(L_2) = K_2$ but $\ff(L_3)$ is not symmetric. See Figure~\ref{fig: ffl2 and k2}.
\end{remark}

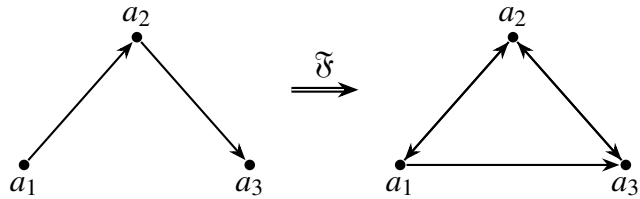
\begin{figure}[ht]
\centering
\begin{tikzpicture}[>=Stealth,thick]
\node[fill=black, circle, minimum size=4pt, inner sep=0pt] (a) at (0,0) {}; 
\node[fill=black, circle, minimum size=4pt, inner sep=0pt] (c) at (3,0) {}; 
\node[fill=black, circle, minimum size=4pt, inner sep=0pt] (b) at (1.5,1.7) {};
\node (A) at (3.4,1) {};
\node (B) at (4.6,1) {};
\node[fill=black, circle, minimum size=4pt, inner sep=0pt] (a1) at (0+5,0) {}; 
\node[fill=black, circle, minimum size=4pt, inner sep=0pt] (c1) at (3+5,0) {}; 
\node[fill=black, circle, minimum size=4pt, inner sep=0pt] (b1) at (1.5+5,1.7) {};
\draw[->] (a) -- (b);
\draw[->] (b) -- (c);
\node[below] at (a){$a_1$};
\node[above] at (b){$a_2$};
\node[below] at (c){$a_3$};

\draw[->] (a1) -- (b1);
\draw[->] (b1) -- (a1);
\draw[->] (b1) -- (c1);
\draw[->] (c1) -- (b1);
\draw[->] (a1) -- (c1);

\node[below] at (a1){$a_1$};
\node[above] at (b1){$a_2$};
\node[below] at (c1){$a_3$};

\draw[double, ->]  (A) -- node [above] {$\ff$} (B);
\end{tikzpicture} 
\caption{An example where $\ff(L_2) = K_2$ but $\ff$ is not symmetric. See Remark~\ref{rem: ffl2 and k2}} \label{fig: ffl2 and k2}
\end{figure}

Similarly to Proposition~\ref{theo:increasing}, we can obtain some information about $\ff$ by applying it to graphs with just two vertices, as described in the following proposition.

\begin{proposition} \label{prop: functor-on-two-point-graphs}
Let $\ff$ be an endofunctor. Then:
\begin{enumerate}
\item $\ff(D_2) \neq D_2$ $\Leftrightarrow$ $\ff = \ff^\mathrm{comp}$.
\item $\ff(K_2) \neq K_2$ $\Leftrightarrow$ $\ff = \ff^\mathrm{disc}$.
\end{enumerate}  
\end{proposition}

\begin{proof}
(1) First notice that if $\ff(D_2) \neq D_2$, then $\ff(D_2) = K_2$. Indeed, suppose $\ff(D_2)$ has just one arrow. Let $p$ be the graph map $p\colon (D_2, a_1, a_2) \to (D_2, a_2, a_1)$. By functoriality, $\ff(D_2)$ must have both arrows $a_1 \to a_2$ and $a_2 \to a_1$.

Now let $G = (V, E) \in \catdigraph$ be a graph with $|G| \geq 2$, and let $v, v' \in G$. Consider the graph map $\phi\colon (D_2, a_1, a_2) \to (G, v, v')$. Applying $\ff$, we obtain $\phi\colon (\ff(D_2), a_1, a_2) \to (\ff(G), v, v')$. Since $\ff(D_2) = K_2$, we have $v \xleftrightarrow{\ff(G)} v'$. Hence, $\ff(G) = K_{(G)}$. 

(2) With the same argument used in the previous item, we can show that if $\ff(K_2) \neq K_2$, then $\ff(K_2) = D_2$. Now suppose there is $G \in \catdigraph$ such that $\ff(G) \neq D_{(G)}$. Let $v, v' \in G$ such that $v \xrightarrow{\ff(G)} v'$. Consider the graph map $\phi\colon (G, v, v') \to (K_2, a_1, a_2)$ given by $\phi(v) = a_1$ and $\phi(x) = a_2$, for any $x \neq v$. By functoriality, we have a graph map $\phi\colon (\ff(G), v, v') \to (D_2, a_1, a_2)$. But then we cannot have $v \xrightarrow{\ff(G)} v'$. This contradiction finishes the proof.
\end{proof}

\begin{corollary} \label{coro:image of complete is complete}
If $\ff \neq \ff^\mathrm{disc}$, then $\ff(K_n) = K_n$, for any $n$.
\end{corollary}
\begin{proof}
Given any $a_i, a_j \in K_n$, we can consider the graph map $\phi\colon (K_2, a_1, a_2) \to (K_n, a_i, a_j)$. By functoriality, $\phi\colon (\ff(K_2), a_1, a_2) \to (\ff(K_n), a_i, a_j)$ is a graph map. Since $\ff(K_2) = K_2$ by Proposition~\ref{prop: functor-on-two-point-graphs}, we have $a_i \xleftrightarrow{\ff(K_n)} a_j$.
\end{proof}

The next proposition is a simple characterization of the transitive closure functor.

\begin{proposition} \label{prop:tc}
Let $\ff\colon \catdigraph\rightarrow\catdigraph^{\mathrm{trans}}$ be a functor such that $\ff(D_2) = D_2$ and $\ff(L_2) = L_2$.
Then, $\ff = \ff^\mathrm{tc}.$
\end{proposition}

\begin{lemma} \label{lemma:partition}
Let $G = (V, E) \in \catdigraph^\mathrm{trans}$. Suppose there exists a pair $(v,v') \notin E$. Then there exists a partition $\{C, C'\}$ of $V$ into two non-empty sets with $v\in C$ and $v'\in C'$ such that $(c,c')\notin E$ for all $c \in C$ and $c' \in C'$.
\end{lemma}
\begin{proof}[Proof of Proposition~\ref{prop:tc}]
Pick any graph $G = (V, E)$. Notice that by Theorem~\ref{theo:increasing}, $\ff$ is arrow increasing.

Assume $v \xrightarrow{\ff^\mathrm{tc}(G)} v'$. Then $v \leadsto v'$ in $G$, and since $\ff$ is arrow increasing, $v \leadsto v'$ in $\ff(G)$. Thus, $\ff^\mathrm{tc} \leq \ff$.

Now assume that $v \nrightarrow v'$ in $\ff^\mathrm{tc}(G)$. 
By Lemma~\ref{lemma:partition} we obtain a partition $\{C, C'\}$ of $V$ with $v \in C$, $v' \in C'$ and the property that $c \nrightarrow c'$ in $\ff^\mathrm{tc}(G)$ for all $c \in C$ and $c'\in C$, which implies that $c \nrightarrow c'$ in $G$, since $G \hookrightarrow \ff^\mathrm{tc}(G)$.

Consider the graph map $\phi\colon (G, v, v') \rightarrow (L_2, a_2, a_1)$ such that $\phi(C) = a_2$ and $\phi(C') = a_1$. Applying $\ff$, we obtain the graph map $\phi\colon (\ff(G), v, v') \to (L_2, a_2, a_1)$. Thus, $v \nrightarrow v'$ in $\ff(G)$. Hence, $\ff(G) = \ff^\mathrm{tc}(G)$. 
\end{proof}

\begin{proof}[Proof of Lemma~\ref{lemma:partition}]
Assume the claim is not true. Then, for any partition $\{C,C'\}$ of $V$ with $v\in C$, $v'\in C'$ there exists some $c\in C$ and $c'\in C'$ with $c \xrightarrow{G} c'$.

Consider first $C_1 = \{ v \}$ and $C_1' = V \backslash C_1$. Let $v_1 \in C_1'$ be such that $v \xrightarrow{G} v_1$.

Now, consider $C_2 = \{ v, v_1 \}$ and $C_2' = V \backslash C_2$. One obtains $v_2 \in C_2'$ such that either $v \xrightarrow{G} v_2$ or $v_1 \xrightarrow{G} v_2$. Since $v \xrightarrow{G} v_1$ and $G$ is transitive, in either case $v \xrightarrow{G} v_2$. 

Recursively define $C_j = \{v, v_1, \ldots, v_{j-1} \}$ and $C_j' = V \backslash C_j$ for $j \geq 1$. At each step we obtain $v_j \in C_j'$ such that $v \xrightarrow{G} v_j$. If $v_j=v'$ for some $j$ we have a contradiction. 

Furthermore, since at each step $v_j \notin C_j$, the process must end when $C_j'$ contains only one element. Thus, at some step in the process we must have $v_j = v'$. 
\end{proof}

\begin{corollary}
If $\ff$ is transitive and arrow increasing, then $\ff^\mathrm{tc} \leq \ff$.
\end{corollary}
\begin{proof}
Given $G \in \catdigraph$, 
\[
v \xrightarrow{\ff^\mathrm{tc}(G)} v' \Leftrightarrow v \leadsto v' \; \text{in $G$} \; \Rightarrow v \leadsto v' \; \text{in $\ff(G)$} \Rightarrow v \xrightarrow{\ff(G)} v'. \qedhere
\]
\end{proof}

\begin{remark}
Not all functors $\ff$ satisfy $\ff(\catdigraph^\mathrm{trans}) \subseteq \catdigraph^\mathrm{trans}$. Take, for example, $\ff^\mathrm{us}$ and $G = \bullet \leftarrow \bullet \rightarrow \bullet$. Then $\ff^\mathrm{us}(G) = \bullet \leftrightarrow \bullet \leftrightarrow \bullet$ which is not transitive.

It turns out that in the case of $\catdigraph^\mathrm{sym}$ we indeed have $\ff(\catdigraph^\mathrm{sym}) \subset \catdigraph^\mathrm{sym}$ for any $\ff \in \funct(\catdigraph, \catdigraph)$. To prove this, we need the following lemma.
\end{remark}

\begin{lemma} \label{lemma: map sym graph}
For any $G = (V, E) \in \catdigraph^\mathrm{sym}$ and $v, v' \in G$, there is a graph map $\phi\colon (G, v, v') \to (G, v', v)$.
\end{lemma}
\begin{proof}
Suppose $G$ has connected components $C_1, \ldots, C_k$ (recall that for symmetric graphs both notions of connectivity coincides).

If $v$ and $v'$ are in different connected components, say $v \in C_1$ and $v' \in C_2$, define $\phi|_{C_1} \equiv v'$ and $\phi|_{C_i} \equiv v$, $i \neq 1$, and we are done.

If $v$ and $v'$ are in the same component, say $v, v' \in C_1$, we define $\phi|_{C_i} \equiv v$, $i \neq 1$, and the problem reduces to defining $\phi$ on $C_1$. Hence, we can suppose that $G$ is connected.

Let $H = (V', E')$ be a connected subgraph of $G$ containing $v$ and $v'$, with the minimum number of vertices possible. It is clear that $H$ is isomorphic to $\ff^\mathrm{us}(L_{n+1})$ for some $n \geq 1$. Let $V' = \{x_0, \ldots, x_n\}$ with $x_0 = v$, $x_n = v'$ and $x_i \xleftrightarrow{H } x_{i+1}$, $i = 0, \ldots, n-1$.

For any $x, y \in V$, define $d(x,y)$ as the number of arrows in the shortest path connecting $x$ to $y$. Let $r\colon G \to H$ be defined by $r(x) = x_k$ where $k = \min\{ d(x, v), n\}$ (notice that $n = d(v, v')$). We claim that $r$ is a graph map. Indeed, let $x, y \in V$ such that $x \xrightarrow{G} y$. Suppose $m = d(x, v) \leq d(y, v) = m'$. If $m = m'$, then $r(x) = r(y)$. If $m' = m + 1 \leq n$, then $r(x) = x_m \xleftrightarrow{H} x_{m+1} = r(y)$. If $m \geq n$, then $r(x) = r(y) = r_n$. This proves the claim.

Finally, let $\phi = \iota \circ f \circ  r\colon G \to G$, where $f\colon H \to H$ is the graph map given by $f(x_i) = x_{n-i}$ and $\iota\colon H \to G$ is the inclusion. Thus, $\phi$ is a graph map satisfying $\phi(v) = v'$ and $\phi(v') = v$.
\end{proof}

\begin{theorem} \label{theo: f sym to sym}
Let $\ff$ be any endofunctor. Then, $\ff(\catdigraph^\mathrm{sym}) \subset \catdigraph^\mathrm{sym}$.
\end{theorem}
\begin{proof}
Let $G \in \catdigraph^\mathrm{sym}$ and $v \xrightarrow{\ff(G)} v'$. Let $\phi\colon (G, v, v') \to (G, v', v)$ be the graph map from Lemma~\ref{lemma: map sym graph}. Applying $\ff$, we obtain the graph map $\phi\colon (\ff(G), v, v') \to (\ff(G), v', v)$, which implies $v' \xrightarrow{\ff(G)} v$.
\end{proof}

\begin{definition} An endofunctor $\ff$ is \emph{additive} if 
\[
\ff \left( \bigsqcup_{i\in I}G_i \right) = \bigsqcup_{i\in I} \ff(G_i)
\]
for all finite collections $\{G_i, i \in I\}$ of graphs.
\end{definition}

\begin{proposition}
All endofunctors $\ff \in \funct(\catdigraph, \catdigraph)$ except $\ff^\mathrm{comp}$ are additive.
\end{proposition}
\begin{proof}

First consider the case $|I| = 2$. Write $G_i = (V_i, E_i)$ and $\ff(G_i) = (V, E_i^{\ff})$ for $i = 1,2$. Also, let $V = V_1 \sqcup V_2$, $E = E_1 \sqcup E_2$, $G = G_1 \sqcup G_2 = (V,E)$ and write $\ff(G)=(V, E^\ff)$. We will prove that $E^\ff = E_1^\ff\sqcup E_2^\ff$.

Let $\phi\colon G \rightarrow D_2$ be the graph map given by $\phi(v_1) = a_1$ for all $v_1 \in V_1$ and $\phi(v_2) = a_2$, for all $v_2 \in V_2$. Since $\ff(D_2) = D_2$ (because $\ff \neq \ff^\mathrm{comp}$), we cannot have $v_1 \xrightarrow{\ff(G)} v_2$ with $v_1 \in V_1$ and $v_2 \in V_2$.
\end{proof}

\section{Representable endofunctors} \label{sec: representable}
Given a family $\Omega$ of graphs, we consider the functor $\ff^\Omega\colon \catdigraph\rightarrow\catdigraph$ defined as follows: given $G = (V, E)$, $\ff^\Omega(G) = (V, E^\Omega)$, where $(v, v')\in E^\Omega$ $\Leftrightarrow$ there exists $\omega \in \Omega$ and a graph map $\phi\colon \omega \to (G, v, v')$ (this means that $v, v' \in \phi(\omega)$, as defined in Section~\ref{sec: background}). Also, set $\ff^\Omega(\phi) = \phi$ for all graph maps $\phi\colon G \to G'$.  

\begin{definition}\label{def: representable functor}
 We say that an endofunctor $\ff$ is \emph{representable} (or \textit{motivic}) whenever there exists a family $\Omega$ of graphs such that $\ff = \ff^\Omega$. In this case we say that $\ff$ is \textit{represented} by $\Omega$, the set of \textit{representers} (or \textit{motifs}) of $\ff$.
\end{definition}

\begin{proposition}
Assume $\ff$ is represented by a family $\Omega$. Then, $\ff^\Omega = \bigcup_{\omega\in \Omega} \ff^{ \{\omega\} }$. Thus, if $\Omega_1 \subseteq \Omega_2$ are two families of graphs, we have $\ff^{\Omega_1} \leq \ff^{\Omega_2}$.
\end{proposition}

\begin{example}
There are several interesting examples of representable functors:

\begin{itemize}

\item Let $\Omega = \{D_1\}$, then $\ff^{\Omega} = \ff^\mathrm{disc}$.

\item Let $\Omega$ be the set of all graphs with vertex set $\{x_1, \ldots, x_n\}$ and such that for each $i = 1, \ldots, n-1$ it holds exactly one of the following conditions: $x_i \to x_{i+1}$ or $x_{i+1} \to x_i$. These graphs are called \textit{zig-zag graphs}. Then $\ff^{\Omega} = \ff^\mathrm{conn}$. 

\item If $\Omega = \{D_2\}$, then $\ff^{\Omega} = \ff^\mathrm{comp}$. 

\item If $\Omega = \{K_2\}$, then $\ff^\Omega = \ff^\mathrm{ls}$.

\item If $\Omega = \{L_2\}$, then $\ff^{\Omega} = \ff^\mathrm{us}$.

\item If $\Omega = \{ \bullet \leftarrow \bullet \rightarrow \bullet \}$ or $\Omega = \{\bullet \rightarrow \bullet \leftarrow \bullet\}$, $\ffo$ is reminiscent of hub or authority type of connections: a connection between two vertices exists if they are both connected to some other vertex (in some direction). In~\cite{Chowdhury2018AFuDoTh} the authors use this kind of connection to construct simplicial complexes over networks.

\item If $\Omega=\{L_{m+1}\}$, then $\ff^\mathrm{us} \circ \ff^{[m]}$.

\item If $\Omega= \{C_n\}$, then $v \seta[\ffo(G)] v'$ when $v$ and $v'$ are in a cycle of size at most $n$. We can also consider $\Omega = \{ C_n \}_{n \in \mathbb{N}}$: then $v \seta[\ffo(G)] v'$ when $v$ and $v'$ are in a cycle (of any length). 
\end{itemize}
\end{example}

The next proposition follows readily from Definition~\ref{def: representable functor}.

\begin{proposition} \label{prop: representable is symmetric}
Every representable functor $\ffo$ is symmetric.
\end{proposition}

\begin{remark} \label{remark: symmetric but not representable}
Not all symmetric endofunctors are representable. Consider, for example, $\ff = \ff^\mathrm{ls} \circ \ff^{[2]}$. Then $\ff(C_4)$ has arrows $a_1 \leftrightarrow a_3$ and $a_2 \leftrightarrow a_4$, that is, $\ff(C_4) = K_2 \sqcup K_2$. See Figure~\ref{fig: symmetric but not representable}. Suppose $\ff = \ffo$, for some family of representers $\Omega$. Then we have a graph map $\phi\colon \omega \to (C_4, a_1, a_3)$ for some $\omega \in \Omega$. But, since there are no more arrows in $\ff(C_4)$, we must have $\phi(\omega) = C_4 \cap \{a_1, a_3\}$, which is a graph with two vertices and no arrows. This implies that $\ff(D_2) = K_2$. By Proposition~\ref{prop: functor-on-two-point-graphs}, we should have $\ff = \ff^\mathrm{comp}$, a contradiction.

\begin{figure}[ht] 
\centering
\begin{tikzpicture}[>=Stealth,thick]
\node[fill=black, circle, minimum size=4pt, inner sep=0pt] (a) at (1,0) {}; 
\node[fill=black, circle, minimum size=4pt, inner sep=0pt] (b) at (0,1) {}; 
\node[fill=black, circle, minimum size=4pt, inner sep=0pt] (c) at (1,2) {};
\node[fill=black, circle, minimum size=4pt, inner sep=0pt] (d) at (2,1) {};
\node (A) at (3,1) {};
\node (B) at (4,1) {};
\node[fill=black, circle, minimum size=4pt, inner sep=0pt] (a1) at (1+5,0) {}; 
\node[fill=black, circle, minimum size=4pt, inner sep=0pt] (b1) at (0+5,1) {}; 
\node[fill=black, circle, minimum size=4pt, inner sep=0pt] (c1) at (1+5,2) {};
\node[fill=black, circle, minimum size=4pt, inner sep=0pt] (d1) at (2+5,1) {};
\draw[->] (a) -- (b);
\draw[->] (b) -- (c);
\draw[->]  (c) -- (d);
\draw[->]  (d) -- (a);

\node[below] at (a){$a_1$};
\node[left] at (b){$a_2$};
\node[above] at (c){$a_3$};
\node[right] at (d){$a_4$};

\draw[->] (a1) -- (c1);
\draw[->] (c1) -- (a1);
\draw[->]  (b1) -- (d1);
\draw[->]  (d1) -- (b1);

\node[below] at (a1){$a_1$};
\node[left] at (b1){$a_2$};
\node[above] at (c1){$a_3$};
\node[right] at (d1){$a_4$};

\draw[double, ->]  (A) -- node [above] {$\ff$} (B);
\end{tikzpicture}
\caption{The image of $C_4$ via $\ff = \ff^\mathrm{ls} \circ \ff^{[2]}$. See Remark~\ref{remark: symmetric but not representable}.} 
\label{fig: symmetric but not representable}
\end{figure}
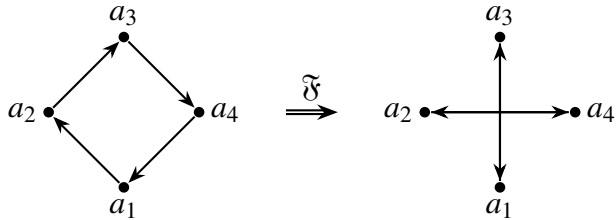
\end{remark}

\begin{definition}
Given an endofunctor $\ff$, we define the \textit{set of completions of $\ff$} by
\[
\mathrm{Comp}(\ff) \coloneqq \{ \omega \in \catdigraph \; | \; \ff(\omega) = K_{(\omega)} \},
\]
which is non-empty since $D_1 \in \mathrm{Comp}(\ff)$.
\end{definition}

\begin{proposition} \label{prop-image-is-complete}
Let $\ff\colon \catdigraph \to \catdigraph$ be any endofunctor and $A = \mathrm{Comp}(\ff)$. Then $
\ff^A \leq \ff^\mathrm{ls} \circ \ff \leq \ff$.
\end{proposition}

\begin{proof}
The second inequality is clear: $\ff^\mathrm{ls} \leq \ff^\mathrm{id}$ implies $\ff^\mathrm{ls} \circ \ff \leq \ff^\mathrm{id} \circ \ff = \ff$. 

To prove the first inequality, let $G \in \catdigraph$ and suppose that $v \xrightarrow{\ff^A(G)} v'$. Then there are $\omega \in A$ and $\phi\colon \omega \to (G, v, v')$. By functoriality, $\phi\colon \ff(\omega) \to (\ff(G), v, v')$ is also a graph map. Since $\ff(\omega) = K_{(\omega)}$, we have $v \xleftrightarrow{\ff(G)} v'$. Thus, $v \xrightarrow{ \ff^\mathrm{ls} \circ \ff (G))} v'$.
\end{proof}

\begin{corollary}
If $\ffo$ is a representable functor, for some family $\Omega$, then $\ffo = \ff^A$.
\end{corollary}
\begin{proof}
We already have $\ff \leq \ffo$. For the other inclusion, just notice that $\Omega \subseteq A$.
\end{proof}

The endofunctors $\ff^\mathrm{comp}$ and $\ff^\mathrm{disc}$ are, respectively, the maximal and minimal elements of the partial order $\leq$, because their outputs are ``all or nothing'' graphs (complete or discrete), that is: $\ff^\mathrm{disc} \leq \ff \leq \ff^\mathrm{comp}$. For other $\ff$, we have two more constraints:

\begin{proposition} \label{prop: order in funct}
Let $\ff \neq \ff^\mathrm{disc}, \ff^\mathrm{comp}$ be an endofunctor. Then 
\[
\ff^\mathrm{disc} \leq \ff^\mathrm{ls} \leq \ff \leq \ff^\mathrm{conn} \leq \ff^\mathrm{comp}.
\]
\end{proposition}
\begin{proof}
Indeed, if $\ff \neq \ff^\mathrm{disc}$, then by Proposition~\ref{prop: functor-on-two-point-graphs} we have $\ff(K_2) = K_2$. Thus, by Proposition~\ref{prop-image-is-complete}, $\ff^{\{K_2\}} = \ff^\mathrm{ls} \leq \ff^{\mathrm{Comp(\ff)}} \leq \ff$.

Now suppose that we don't have $\ff \leq \ff^\mathrm{conn}$. Then, for some $G \in \catdigraph$, we have $v \nrightarrow v'$ in $\ff^\mathrm{conn}(G)$ and $v \xrightarrow{\ff(G)} v'$. Then $v$ and $v'$ are in different weak connected components of $G$, say $C$ and $C'$. Thus, the surjective map $\phi\colon G \to D_2$ given by $\phi(C) = v$ and $\phi(G \setminus C) = v'$ is a graph map and, by functoriality, $\ff(D_2) \neq D_2$. Thus, $\ff = \ff^\mathrm{comp}$.
\end{proof}

\section{Pointed representable functors} \label{sec: pointed representable}

For a representable functor $\ffo$, we have $v \xrightarrow{\ffo(G)} v'$ when there are $\omega \in \Omega$ and a graph map $\phi\colon \omega \to (G, v, v')$, that is, $v, v' \in \phi(\omega)$. Now we will adapt this definition in order to obtain, in some cases, non-symmetric endofunctors. In order to do this, we will distinguish which vertices $z, \widehat{z}$ of $\omega$ satisfy $\phi(z) = v$ and $\phi(\widehat{z}) = v'$. 

\begin{definition}
Let $\catdigraph^\ast$ be the category whose objects are all triples $(\omega, z, \widehat{z})$, called \textit{pointed graphs}, where $\omega$ is a graph with $z, \widehat{z} \in \omega$. The morphisms in $\catdigraph^\ast$ from $(\omega_1, z_1, \widehat{z}_1)$ to $(\omega_2, z_2, \widehat{z}_2)$ are all graph maps $\phi\colon (\omega_1, z_1, \widehat{z}_1) \to (\omega_2, z_2, \widehat{z}_2)$.

Given a set $\oa$ of pointed graphs, consider $\ffoa \in \funct(\catdigraph, \catdigraph)$, where $\ffoa(G)$ satisfies, for any $G \in \catdigraph$,
\[
v \xrightarrow{\ffoa(G)} v' \Leftrightarrow \exists (\omega, z, \widehat{z}) \in \Omega^\ast, \exists \phi\colon (\omega, z, \widehat{z}) \to (G, v, v' ).
\] 

When $\ff = \ffoa$ for some $\oa \subset \catdigraph^\ast$, we say that $\ff$ is \textit{pointed representable}, and is \textit{represented} by $\oa$.
\end{definition}

\begin{example}

If $\oa$ consists exactly of:
\begin{itemize}
\item $(\omega, z, z)$, then $\ffoa = \ff^\mathrm{disc}$, for any $\omega \in \catdigraph$ and $z \in \omega$. 
\item $(D_2, a_1, a_2)$, then $\ffoa = \ff^\mathrm{comp}$.
\item $(L_2, a_1, a_2)$, then $\ffoa = \ff^\mathrm{id}$, the identity functor. Remember that $L_2$ has vertex set $\{a_1, a_2\}$ and an arrow $a_1 \to a_2$.
\item $(L_2, a_2, a_1)$, then $\ffoa = \ff^\mathrm{rev}$.
\item $(K_2, a_1, a_2)$, then $\ffoa = \ff^\mathrm{ls}$.
\item $\{(L_n, a_1, a_n), n \leq m\}$, for a given $m \in \mathbb{N}$, then $\ffoa = \ff^{[m]}$.
\item $\{(L_n, a_1, a_n), n \in \mathbb{N}\}$, then $\ffoa = \ff^\mathrm{tc}$. 
\end{itemize}
\end{example}

Any representable functor $\ffo$ is equal to a pointed representable functor, just taking as representer the set of all possible pointed graphs induced by $\Omega$. More precisely:

\begin{proposition} \label{prop: pointed and not pointed}
Let $\Omega$ be a family of graphs. Consider the family of pointed graphs induced by $\Omega$, $P^\ast(\Omega) \coloneqq \{(\omega, z, \widehat{z}), z, \widehat{z} \in \omega \in \Omega\}$. Then $\ff^\Omega = \ff^{P^\ast(\Omega)}$.
\end{proposition}
\begin{proof}
It is clear that $\ff^{P^\ast(\Omega)} \leq \ffo$. Now we prove the remaining inequality. Given $G \in \catdigraph$, by definition, $v \xrightarrow{\ff(G)} v' \Leftrightarrow \exists \omega \in \Omega, \exists \phi\colon \omega \to (G, v, v')$. This is equivalent to the condition that there is $z, \widehat{z} \in \omega$ such that $\phi\colon (\omega, z, \widehat{z}) \to (G, v, v')$. Since $(\omega, z, \widehat{z}) \in \oa$, we have $v \xrightarrow{\ff^{P^\ast(\Omega)}(G)} v'$. Thus $\ffo \leq \ff^{P^\ast(\Omega)}$. 
\end{proof}

Let's consider some standard notation that will ease the reading of what follows: for any graph $G = (V, E)$, the maps $s, t\colon E \to V$ are defined by $s(e) = v$ and $t(e) = v'$, for any $e = (v, v') \in E$. These functions are called the \textit{source} and \textit{target}, respectively. Thus, for any $e \in E$ we have $s(e) \seta[G] t(e)$. 

When computing the composition $\ff^{\oa_2} \circ \ff^{\oa_1}$, for two families of pointed graphs $\oa_1$ and $\oa_2$, we encounter the following construction. Let $(\omega, z, \widehat{z}) \in \oa_2$. For every arrow $e$ in $\omega$, one ``attaches'' to $\omega$ an arbitrary graph $(\omega_{e}, z_{e}, \widehat{z}_{e}) \in \oa_1$ by identifying $z_{e}$ with $s(e)$ and $\widehat{z}_{e}$ with $t(e)$. After doing this to all the arrows in $\omega$, we delete the original arrows of $\omega$ and only keep the ``attached'' ones, and thus obtain another graph, let's say $Z$. See Figure~\ref{fig:composition} for an illustration. 

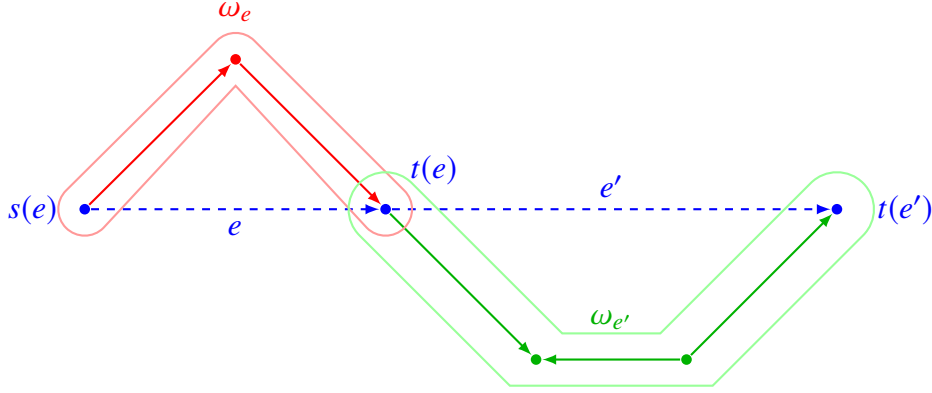
\begin{figure}

\centering

\begin{tikzpicture}[>=Stealth,thick,scale=2]
\node[fill=blue, circle, minimum size=4pt, inner sep=0pt] (w1) at (180:2cm){}; 
\node[fill=blue, circle, minimum size=4pt, inner sep=0pt] (w2) at (0,0){}; 
\node[fill=blue, circle, minimum size=4pt, inner sep=0pt] (w3) at (0:3cm){}; 
\node[fill=red, circle, minimum size=4pt, inner sep=0pt] (t1) at (135:1.41cm){}; 
\node[fill=green!70!black, circle, minimum size=4pt, inner sep=0pt] (t3) at (1,-1){}; 
\node[fill=green!70!black, circle, minimum size=4pt, inner sep=0pt] (t4) at (2,-1){}; 
\node[left=6pt,blue] at (w1){$s(e)$};
\node[above right=5pt,blue] at (w2){$t(e)$};
\node[right=11pt,blue] at (w3){$t(e')$};
\node[above=10pt,red] at (t1){$\omega_{e}$};
\node[above=7pt,green!70!black] at (1.5, -1){$\omega_{e'}$};

\draw[blue,dashed,-latex] (w1) --node[below]{$e$} (w2);
\draw[blue,dashed,-latex] (w2) --node[above]{$e'$} (w3);

\draw[red,-latex] (w1) -- (t1);
\draw[red,-latex] (t1) -- (w2);
\draw[green!70!black,-latex] (w2) -- (t3);
\draw[green!70!black,-latex] (t4) -- (t3);
\draw[green!70!black,-latex] (t4) -- (w3);

\begin{scope}[red!40,line join=rounded]
\draw
([shift={(-45:5pt)}]w1.center) arc (-45:-225:5pt) --
([shift={(135:5pt)}]t1.center) arc (135:45:5pt) --
([shift={(45:5pt)}]w2.center)  arc (45:-135:5pt) --
([shift={(-90:5pt)}]t1.center) -- cycle
;
\end{scope}

\begin{scope}[green!40,line join=rounded]
\draw
([shift={(-135:7pt)}]w2.center) arc (225:45:7pt) --
([shift={(45:7pt)}]t3.center) --
([shift={(135:7pt)}]t4.center) --
([shift={(135:7pt)}]w3.center) arc (135:-45:7pt) --
([shift={(-45:7pt)}]t4.center) --
([shift={(-135:7pt)}]t3.center) -- cycle
;
\end{scope}

\end{tikzpicture}

\caption{In blue: the graph $\omega$, with dashed arrows. In red: the graph $\omega_{e}$. In green: the graph $\omega_{e'}$. The graph $Z$ is obtained excluding from the picture the blue dashed arrows. See Definition~\ref{defi: exclamation point}.} \label{fig:composition}
\end{figure}

\begin{definition} \label{defi: exclamation point}
Let $(\omega, z, \widehat{z}) \in \catdigraph^\ast$ with vertex set $V_\omega$ and arrow set $E_\omega$, and let $\oa_1$ be a set of pointed graphs. Let $F\colon E_\omega \to \oa_1$ be any function assigning edges from $E_\omega$ to pointed graphs in $\oa_1$.  For each $e\in E_\omega$ denote by $(\omega_{e}, z_{e}, \widehat{z}_{e})$ the element $F(e)$ and denote its  vertex set by  $V_{\omega_e}$. 
Define a graph $\omega_F$ as follows:
\begin{enumerate}
\item The vertex set of $\omega_F$ is the quotient of the disjoint union
\[
\left( V_\omega \bigsqcup_{e \in E_\omega} V_{\omega_e} \right) / {\sim}
\]
where for every $e \in E_\omega$ we identify $s(e) \sim z_e$ and $t(e) \sim \widehat{z}_e$.
\item There is an arrow from $x$ to $y$ in $\omega_F$ $\Leftrightarrow$ $x$ and $y$ are both contained in some $\omega_e$ and $x \seta[\omega_e] y$.
\end{enumerate}

Since $z, \widehat{z} \in \omega_F$, we can consider the pointed graph $(\omega_F, z, \widehat{z})$.
Notice that for all $e \in E_\omega$ we have natural inclusions $\iota_{e}\colon (\omega_{e}, z_{e}, \widehat{z}_{e}) \hookrightarrow (\omega_F, s(e), t(e))$.

Define the \emph{composition} of the single pointed graph $(\omega,z,\widehat{z})$ with the whole set $\oa_1$ as
\[
(\omega, z, \widehat{z})  \diamond \oa_1 \coloneqq \{(\omega_F, z, \widehat{z}) \; | \; F\colon E_\omega \to \oa_1 \text{ is a function} \}.
\] 

Finally, define 
\[
\oa_2 \diamond \oa_1 \coloneqq \bigcup_{(\omega, z, \widehat{z}) \in \oa_2} (\omega, z, \widehat{z}) \diamond \oa_1,
\]
which we will refer to as \textit{the composition of $\oa_2$ with $\oa_1$}.
\end{definition}

\begin{theorem} \label{theo: composition of pointed representable}
Composition of representable functors is also representable. More precisely, given $\oa_1$ and $\oa_2$ two families of pointed graphs, we have
\[
\ff^{\oa_2} \circ \ff^{\oa_1} = \ff^{\oa_2 \diamond \oa_1}.
\]
\end{theorem}
\begin{proof}
Let $\ff_1 = \ff^{\oa_1}$, $\ff_2 = \ff^{\oa_2}$, $\ff = \ff^{\oa_2 \diamond \oa_1}$ and $G = (V, E) \in \catdigraph$.

Suppose $v \xrightarrow{\ff_2 (\ff_1(G))} v'$. By definition, there are $(\omega, z, \widehat{z}) \in \oa_2$ and a graph map $\phi\colon (\omega, z, \widehat{z}) \to (\ff_1(G), v, v')$. Denote the vertex set of $\omega$ by $V_\omega$ and the arrow set by $E_\omega$. Since $\phi$ is a graph map, for each $e \in E_\omega$ we have $\phi(s(e)) \seta[\ff_1(G)] \phi(t(e))$. Then, for every $e \in E_\omega$ there are $(\omega_e, z_{e}, \widehat{z}_{e}) \in \oa_1$ and a graph map $\phi_{e}\colon (\omega_{e}, z_{e}, \widehat{z}_{e}) \to (G, \phi(s(e)), \phi(t(e)))$. Let $F\colon E_\omega \to \oa_1$ be defined by $F(e) = (\omega_{e}, z_{e}, \widehat{z}_{e})$.

The graph maps $\phi$ and $\phi_{e}$ induce a graph map $\psi\colon (\omega_F, z, \widehat{z}) \to (G, v, v')$ (notice that if $x \in \omega_{e}$ and $y \in \omega_{e'}$ with $e \neq e'$, then there is no arrow between $x$ and $y$ in $\omega_F$). This implies that $v \xrightarrow{\ff(G)} v'$, and then $\ff_2 \circ \ff_1 \leq \ff$. 

Reciprocally, if $v \xrightarrow{\ff(G)} v'$, there is a graph map $\psi\colon (\omega_F, z, \widehat{z}) \to (G, v, v')$, for some $(\omega, z, \widehat{z}) \in \oa_2$ with vertex set $V_\omega$ and arrow set $E_\omega$, and for some $F\colon E_\omega \to \oa_1$, where $F(e) = (\omega_{e}, z_{e}, \widehat{z}_{e})$.

Consider the graph maps $\phi_{e}\colon (\omega_{e}, z_{e}, \widehat{z}_{e}) \to (G, \psi(s(e)), \psi(t(e)))$ obtained by restricting the domain of $\psi$, that is: $\phi_{e} = \psi|_{\omega_{e}}$. Thus, $\phi_{e}(s(e)) \xrightarrow{\ff_1(G)} \phi_{e}(t(e))$. Consider $\phi\colon (\omega, z, \widehat{z}) \to (\ff_1(G), v, v')$ defined by $\phi(x) = \psi(x)$. This is indeed a graph map since for any $e \in E_\omega$ we have 
\[
\phi(s(e)) = \phi_e(s(e)) \xrightarrow{\ff_1(G)} \phi_e(t(e)) = \phi(t(e)).
\]
Thus, $v \xrightarrow{\ff_2 \circ \ff_1(G)} v'$, that is: $\ff \leq \ff_2 \circ \ff_1$.
\end{proof}

\begin{example} \label{ex:: pointed}
In Remark~\ref{remark: symmetric but not representable} we have seen that $\ff = \ff^\mathrm{ls} \circ \ff^{[2]}$ is not representable. However, this endofunctor is pointed representable: $\ff^{[2]}$ is pointed represented by a line graph with 3 vertices, $\oa_1 = \{(L_3, a_1, a_3)\}$, and $\ff^\mathrm{ls}$ is pointed represented by the complete graph with two vertices, $\oa_2 = \{(\omega, z, \widehat{z})\}$. By the Theorem~\ref{theo: composition of pointed representable}, $\ff$ is pointed represented by $\oa_2 \diamond \oa_1 = \{(Z, z, \widehat{z})\}$, with $(Z, z, \widehat{z})$ isomorphic to $(C_4, a_1, a_3)$. See Figure~\ref{fig:composition example}.

\begin{figure}[ht]

\centering

\begin{tikzpicture}[scale=2, >=Stealth, thick]

\node[fill=black, circle, minimum size=4pt, inner sep=0pt] (a1) at (-1, 0){}; 
\node[fill=black, circle, minimum size=4pt, inner sep=0pt] (a2) at (0, 1){}; 
\node[fill=black, circle, minimum size=4pt, inner sep=0pt] (a3) at (1, 0){}; 
\node[fill=black, circle, minimum size=4pt, inner sep=0pt] (a4) at (0, -1){}; 
\node[left] at (a1){$z$};
\node[right] at (a3){$\widehat{z}$};
\draw[blue,dashed,-latex] (a1) to [out=20,in=160] (a3);
\draw[blue,dashed,-latex] (a3) to [out=-160,in=-20] (a1);

\draw[red,-latex] (a1) -- (a2);
\draw[red,-latex] (a2) -- (a3);
\draw[red,-latex] (a3) -- (a4);
\draw[red,-latex] (a4) -- (a1);
\end{tikzpicture}

\caption{In blue: the graph $(\omega, z, \widehat{z})$. In red: two copies of $L_3$. The graph $Z$ is obtained excluding from the picture the blue arrows. See Example~\ref{ex:: pointed}.} \label{fig:composition example}

\end{figure}
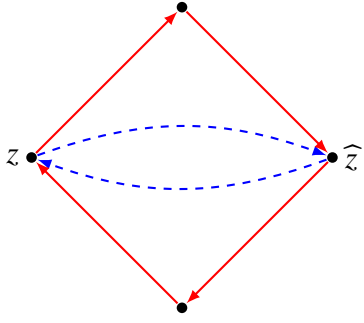

\end{example}

The interesting part of Theorem~\ref{theo: composition of pointed representable} is the formula for the composite, since it turns out that \textit{every} endofunctor is pointed representable.

For any endofunctor $\ff$, let 
\[
\Omega^\ast_\ff \coloneqq \{(\omega, z, \widehat{z}) \in \catdigraph^\ast \; | \; z \seta[\ff(\omega)] \widehat{z}\}.
\]

\begin{remark}[Interpretation of $\Omega^\ast_\ff$]
 The set $\Omega^\ast_\ff$ is composed of all graphs $(\omega, z, \widehat{z})$ such that $\ff(\omega)$ has at least one arrow, together with all arrows $z \to \widehat{z}$ generated by $\ff$. Intuitively, $\Omega^\ast_\ff$ is the collection of all possible arrows that $\ff$ can create, thus it is not surprising that all the information of $\ff$ is contained in $\Omega^\ast_\ff$ as stated by the following theorem.
\end{remark}
 
\begin{theorem} \label{theo: every functor is pointed representable}
For every endofunctor $\ff$, it holds $\ff = \ff^{\Omega^\ast_\ff}$.
\end{theorem}
\begin{proof}
Let $G$ be a graph. If $v \xrightarrow{\ff^{\Omega^\ast_\ff}(G)} v'$, then, by definition, there are $(\omega, z, \widehat{z}) \in \Omega^\ast_\ff$ and a graph map $\phi\colon (\omega, z, \widehat{z}) \to (G, v, v')$. Applying $\ff$ to this diagram, we obtain the graph map $\phi\colon (\ff(\omega), z, \widehat{z}) \to (\ff(G), v, v')$. By definition of $\Omega^\ast_\ff$, we have that $z \xrightarrow{\ff(\omega)} \widehat{z}$, and since $\phi$ is a graph map we also have $v \xrightarrow{\ff(G)} v'$.

Reciprocally, if $v \xrightarrow{\ff(G)} v'$, then $(G, v, v') \in \Omega^\ast_\ff$ and since the identity $\mathrm{id}_G\colon (G, v, v') \to (G, v, v')$ is a graph map, it follows that $v \xrightarrow{\ff^{\Omega^\ast_\ff}(G)} v'$.
\end{proof}

The set $\Omega^\ast_\ff$ is often unnecessarily large. For example, if $\ff = \ff^\mathrm{comp}$, then $\Omega^\ast_\ff$ is the set of all pointed graphs. Reducing its size is the content of the next section.

\section{Simplification and clustering} \label{sec: simplifications}
Given a set $\oa$ of pointed graphs, can one find a simpler (in some suitable sense) $\oa_0$ such that $\ffoa = \ff^{\oa_0}$? A necessary and sufficient condition for this is the following.

\begin{definition}
We say that $(\omega_2, z_2, \widehat{z}_2)$ \textit{covers} $(\omega_1, z_1, \widehat{z}_1)$ if there is a graph map $p\colon (\omega_2, z_2, \widehat{z}_2) \to (\omega_1, z_1, \widehat{z}_1)$. Given two families of graphs $\oa_1$ and $\oa_2$, we say that $\oa_2$ \textit{covers} $\oa_1$ if for any $(\omega_1, z_1, \widehat{z}_1) \in \oa_1$, there is $(\omega_2, z_2, \widehat{z}_2) \in \oa_2$ that covers $(\omega_1, z_1, \widehat{z}_1)$. We denote this by $\oa_1 \preccurlyeq \oa_2$.

Notice that $\oa_1 \subseteq \oa_2$ implies $\oa_1 \preccurlyeq \oa_2$.
\end{definition}

For the next two results, let $\oa_1$ and $\oa_2$ be two families of pointed graphs, and let the endofunctors they represented be denoted by $\ff_1 = \ff^{\oa_1}$ and $\ff_2 = \ff^{\oa_2}$, respectively.

\begin{theorem} \label{theo: ff_1 leq ff_2 se e só se oa_1 menor oa_2}
$\ff_1 \leq \ff_2 \Leftrightarrow \oa_1 \preccurlyeq \oa_2$.
\end{theorem}
\begin{proof}
Let's prove the ($\Leftarrow$) implication first. Suppose $v \xrightarrow{\ff_1(G)} v'$. Then there are $(\omega_1, z_1, \widehat{z}_1) \in \oa_1$ and a graph map $\phi\colon (\omega_1, z_1, \widehat{z}_1) \to (G, v, v')$. Since $\oa_1 \preccurlyeq \oa_2$, there are $(\omega_2, z_2, \widehat{z}_2) \in \oa_2$ and a graph map $p\colon (\omega_2, z_2, \widehat{z}_2) \to (\omega_1, z_1, \widehat{z}_1)$. 

\begin{displaymath}
\begin{xymatrix}{
(\omega_2, z_2, \widehat{z}_2) \ar[d]^p \ar[dr]^{\phi \circ p} \\
(\omega_1, z_1, \widehat{z}_1) \ar[r]^\phi & (G, v, v').
}
\end{xymatrix}
\end{displaymath}

Then $\phi \circ p\colon (\omega_2, z_2, \widehat{z}_2) \to (G, v, v')$ is a graph map, which implies $v \xrightarrow{\ff_2(G)} v'$. Hence, $\ff_1 \leq \ff_2$.

Now for the ($\Rightarrow$) implication, let $(\omega_1, z_1, \widehat{z}_1) \in \oa_1$. It is clear that $z_1 \xrightarrow{\ff_1(\omega_1)} \widehat{z}_1$, and $z_1 \xrightarrow{\ff_2(\omega_1)} \widehat{z}_1$ follows from the fact that $\ff_1 \leq \ff_2$. By definition, there are $(\omega_2, z_2, \widehat{z}_2) \in \oa_2$ and a graph map $p\colon (\omega_2, z_2, \widehat{z}_2) \to (\omega_1, z_1, \widehat{z}_1)$. This proves that $\oa_1 \preccurlyeq \oa_2$.
\end{proof}

\begin{corollary}
We have $\ff_1 = \ff_2$ $\Leftrightarrow$ $\oa_1 \preccurlyeq \oa_2$ and $\oa_2 \preccurlyeq \oa_1$.
\end{corollary}

Via Proposition~\ref{prop: pointed and not pointed} we can obtain a theorem similar to Theorem~\ref{theo: ff_1 leq ff_2 se e só se oa_1 menor oa_2} which applies in the case of standard (non pointed) representable functors.

\begin{definition}
Let $\Omega_1$ and $\Omega_2$ be any two families of graphs. We say that $\Omega_2$ \textit{covers} $\Omega_1$ if for any $\omega_1 \in \Omega_1$ and $z_1, \widehat{z}_1 \in \omega_1$, there are $\omega_2 \in \Omega_2$ and a graph map $p\colon \omega_2 \to (\omega_1, z_1, \widehat{z}_1)$. We denote this by $\Omega_1 \preccurlyeq \Omega_2$. Notice that $\Omega_1 \preccurlyeq \Omega_2$ is equivalent to $P^\ast(\Omega_1) \preccurlyeq P^\ast(\Omega_2)$, where $P^\ast$ is as in Proposition~\ref{prop: pointed and not pointed}.
\end{definition}

\begin{corollary}
Let $\Omega_1$ and $\Omega_2$ be two families of graphs. Then $\ff^{\Omega_1} \leq \ff^{\Omega_2}$ $\Leftrightarrow$ $\Omega_1 \preccurlyeq \Omega_2$.
\end{corollary}
\begin{proof}
We already know that $\Omega_1 \preccurlyeq \Omega_2 \Leftrightarrow {P^\ast(\Omega_1)} \preccurlyeq {P^\ast(\Omega_2)}$. By Proposition~\ref{prop: pointed and not pointed}, $\ff^{\Omega_1} = \ff^{P^\ast(\Omega_1)}$ and $\ff^{\Omega_2} = \ff^{P^\ast(\Omega_2)}$. 
Then,
\[
\ff^{\Omega_1} \leq \ff^{\Omega_2} \Leftrightarrow \ff^{P^\ast(\Omega_1)} \leq \ff^{P^\ast(\Omega_2)} \Leftrightarrow {P^\ast(\Omega_1)} \preccurlyeq {P^\ast(\Omega_2)} \Leftrightarrow \Omega_1 \preccurlyeq \Omega_2,
\]
where the second equivalence follows by Theorem~\ref{theo: ff_1 leq ff_2 se e só se oa_1 menor oa_2}.
\end{proof}

\begin{example}
Given $\oa \subset \catdigraph^\ast$ and $(\omega_0, z_0, \widehat{z}_0) \in \oa$, let 
\[
A_0 \coloneqq \{(\omega, z, \widehat{z}) \in \oa \; | \; (\omega_0, z_0, \widehat{z}_0) \text{ covers } (\omega, z, \widehat{z}) \},
\]
and $\oa_0 = (\oa \setminus A_0) \cup \{(\omega, z, \widehat{z})\} $. Then $\oa_0 \subseteq \oa$ implies $\oa_0 \preccurlyeq \oa$, and, moreover, we have $\oa \preccurlyeq \oa_0$. Thus, $\ffoa = \ff^{\oa_0}$. We can regard $\oa_0$ as a simplification of $\oa$ which yields the same pointed representable functor. 

Consider, for example, $\oa = \{(L_n, a_1, a_n), n \leq m\}$, for a given $m \in \mathbb{N}$. Then, choosing $(L_m, a_1, a_m)$ as the pointed graph $(\omega_0, z_0, \widehat{z}_0)$ above, we obtain $\oa_0 = \{(L_m, a_1, a_m)\}$ and $\ff^{[m]} = \ffoa = \ff^{\oa_0}$, that is: we can ``throw away'' every $(L_n, a_1, a_n)$ with $n < m$ and still obtain the same functor. When $\oa = \{(L_n, a_1, a_n), n \in \mathbb{N}\}$, however, we cannot obtain a finite family $\oa_0$ such that $\ffoa = \ff^{\oa_0}$.

Another example: take $\oa$ as the set of all pointed graphs, and $\oa_0 = \{D_2\}$. We have $\oa_0 \preccurlyeq \oa$ because $\oa_0 \subset \oa$, and $\oa \preccurlyeq \oa_0$, since for any $(\omega, z, \widehat{z}) \in \oa$ the map $\phi\colon (D_2, a_1, a_2) \to (\omega, z, \widehat{z})$ is a graph map. Thus, $\ffoa = \ff^{\oa_0} = \ff^\mathrm{comp}$.
\end{example}

\subsection{Clustering functors}

We already know (Proposition~\ref{prop: representable is symmetric}) that any representable functor $\ffo$ is symmetric. We now establish conditions under which $\ffo$ turns out to be a clustering functor (that is, $\ffo$ is also transitive). 

\begin{definition}
For any $G_1, G_2 \in \catdigraph$, $v_1 \in G_1$ and $v_2 \in G_2$, define the \textit{wedge product} of $(G_1, v_1)$ and $(G_2, v_2)$ by
\[
(G_1, v_1) \vee (G_2, v_2) \coloneqq \dfrac{G_1 \sqcup G_2}{v_1 \sim v_2},
\]
that is: we identify $v_1$ and $v_2$ in the disjoint union $G_1 \sqcup G_2$. 

Given two families of graphs $\Omega_1$ and $\Omega_2$, define 
\[
\Omega_1 \vee \Omega_2 \coloneqq \bigcup \limits_{\substack{\omega_i \in \Omega_i, \; z_i \in \omega_i
\\
i = 1, 2}}(\omega_1, z_1) \vee (\omega_2, z_2).
\]

We say that $\Omega$ is \textit{wedge covered} whenever $\Omega \vee \Omega \preccurlyeq \Omega$.
\end{definition}

\begin{theorem} \label{theo: wedge covered families}
$\Omega$ is a wedge covered family of graphs $\Leftrightarrow$ $\ffo$ is a clustering functor.
\end{theorem}
\begin{proof}
Let's prove the ($\Rightarrow$) implication first. Let $G$ be a graph. Suppose that $v_1 \xrightarrow{\ffo(G)} v_2$ and $v_2 \xrightarrow{\ffo(G)} v_3$. Then, for some $\omega_1, \omega_2 \in \Omega$, there exist graph maps $\phi_1\colon (\omega_1, z_1, \widehat{z}_1) \to (G, v_1, v_2)$ and $\phi_2\colon (\omega_2, z_2, \widehat{z}_2) \to (G, v_2, v_3)$. Let $\omega = (\omega_1, \widehat{z}_1) \vee (\omega_2, z_2)$. Then the map $\phi\colon (\omega, z_1, \widehat{z}_2) \to (G, v_1, v_3)$ given by 
\[
\phi(v) = 
\begin{cases}
\phi_1(v), & \text{if } v \in \omega_1, \\
\phi_2(v), & \text{if } v \in \omega_2,
\end{cases} 
\]
is a well defined graph map. Since $\Omega$ is wedge covered, there are $\tilde{\omega} \in \Omega$ and a graph map $p\colon (\tilde{\omega}, z, \widehat{z}) \to (\omega, z_1, \widehat{z}_2)$. The composite $\phi \circ p\colon (\tilde{\omega}, z, \widehat{z}) \to (G, v_1, v_3)$ is a graph map, which implies $v_1 \xrightarrow{\ffo(G)} v_3$, that is, $\ffo$ is transitive. That $\ffo$ is symmetric follows from it being representable.

Now let's prove the ($\Leftarrow$) implication. Let $z_1 \in \omega_1 \in \Omega$ and $z_2 \in \omega_2 \in \Omega$. Consider $G = (\omega_1, z_1) \vee (\omega_2, z_2)$. Let $v, v' \in G$. Consider the inclusions $i_1\colon \omega_1 \hookrightarrow G$ and $i_2\colon \omega_2 \hookrightarrow G$. Suppose $v \in i_1(\omega_1)$ and $v' \in i_2(\omega_2)$. Then we have $v \xrightarrow{\ff(G)} i_1(z_1)$ and $i_2(z_2) \xrightarrow{\ff(G)} v'$. Since $\ff(G)$ is transitive and $i_1(z_1) = i_2(z_2)$, we have $v \xrightarrow{\ff(G)} v'$. By definition of wedge covered, there are $\omega \in \Omega$ and a graph map $p\colon \omega \to (G, v, v')$. The cases when both $v, v' \in i_1(\omega_1)$ or $v, v' \in i_2(\omega_2)$ are immediate. This proves that $\Omega$ is wedge covered.
\end{proof}

\begin{example} \label{rec-e-nao-rec}
Two important examples of wedge covered families are:
\begin{enumerate}
\item the set of arbitrary size \textit{reciprocal} line graphs: $\Omega = \{\ff^\mathrm{us}(L_n), n \in \mathbb{N}\}$. $\ffo$ is called \textit{the reciprocal clustering functor}, denoted by $\ff^\mathrm{rec}$. Notice that $\ff^\mathrm{us}(L_n)$ has vertex set $\{a_1, \dots, a_n\}$ and arrows $a_i \leftrightarrow a_{i+1}$, for $i = 1, \dots, n-1$.

\item the set of arbitrary size cycle graphs $\Omega = \{C_n, n \in \mathbb{N}\}$. The functor  $\ffo$ is called \textit{the non-reciprocal clustering functor}, denoted by $\ff^\mathrm{nrec}$.
\end{enumerate}
\end{example}

\section{Relations with hierarchical clustering} \label{sec: relations to hc}

In this section we will explore how functors $\ff\colon \catgraph \to \catgraph$ induce functors on the category of networks (and, as special case, metric spaces). This will allow us to create clustering functors which can be visualized as \textit{treegrams}.

\subsection{Hierarchical clustering of extended networks} \label{subsec: hierarchical clustering}

We now introduce a suitable notion of network.

\begin{definition}
An \textit{extended network} is a pair $\bx = (X, w_X)$, where $X$ is a finite set, and $w_X\colon X \times X \to \mathbb{R} \cup \{ + \infty \}$, with $w_X(x, x) < +\infty$, for any $x \in X$, is a function (called \textit{weight function}). A \textit{network map} $\phi\colon \bx \to \by$, where $\by = (Y, w_Y)$, is a map $\phi\colon X \to Y$ such that $w_Y(\phi(x), \phi(x')) \leq w_X(x, x')$ for any $x, x' \in X$, with the convention that $a < + \infty$, for all $a \in \mathbb{R}$.
\end{definition}

Denote by $\netex$ the category of extended networks and network maps. Dropping the prefix ``extended'' means that the codomain of the weight function is $\mathbb{R}$, instead of $\mathbb{R} \cup \{+\infty\}$. Thus, the category $\net$ of networks (see~\cite{Memoli2017DiIs}) is a subcategory of $\netex$.

For any $\bx = (X, w_X) \in \netex$ and $\epsilon \in \mathbb{R}$, define the graph $\bx_\epsilon$ with vertex set 
\[
V(\bx_\epsilon) \coloneqq \{ x \in X \; | \; w_X(x, x) \leq \epsilon \}
\] 
and arrow set 
\[
E(\bx_\epsilon) \coloneqq \{ (x, x') \in V(\bx_\epsilon) \times V(\bx_\epsilon) \; | \; w_X(x, x') \leq \epsilon \}.
\]

Any functor $\ff\colon \catdigraph \to \catdigraph$ induces a functor $\ffh\colon \netex \to \netex$ defined on the objects of $\netex$ by $\ffh(\bx) = (X, w_X^\ff)$, where 
\[
w_X^\ff(x, x') \coloneqq \min \left\{ \epsilon \in \mathbb{R} \; | \; x \seta[\ff(\bx_\epsilon)] x'\right\},
\]
and where we consider the minimum over the empty set to be $+ \infty$. This is equivalent to stating  that $\ffh(\bx)_\epsilon \coloneqq \ff(\bx_\epsilon)$. 
Now let $\phi\colon \bx \to \by$ be a network map, with $\by = (Y, w_Y)$. On morphisms, define $\ffh(\phi) \coloneqq \phi$. Let's prove that $\phi\colon \ffh(\bx) \to \ffh(\by)$ is a network map. 

It is clear that $\phi$ induces graph maps $\phi_\epsilon\colon \bx_\epsilon \to \by_\epsilon$ for all $\epsilon \in \R$, because 
\[
x \seta[\bx_\epsilon] x' \Leftrightarrow
w_X(x, x') \leq \epsilon \Rightarrow
w_Y(\phi(x), \phi(x')) \leq \epsilon \Leftrightarrow
\phi(x) \seta[\by_\epsilon] \phi(x').
\]
The functoriality of $\ff$ guarantees that $\phi_\epsilon\colon \ff(\bx_\epsilon) \to \ff(\by_\epsilon)$ is also a graph map, and the following
\[
w_X^\ff(x, x') \leq \epsilon  \Leftrightarrow x \seta[\ff(\bx_\epsilon)] x'\Rightarrow
\phi(x) \seta[\ff(\by_\epsilon)] \phi(x') \Leftrightarrow w_Y^\ff(\phi(x), \phi(x')) \leq \epsilon
\]
implies that $\phi\colon \ffh(\bx) \to \ffh(\by)$ is a network map.

An extended network $\bx = (X, w_X)$ is said to be \textit{symmetric} if $w_X$ is symmetric, that is, $w_X(x, x') = w_X(x', x)$, for any $x, x' \in X$. We denote the category of symmetric extended networks by $\netex^\mathrm{sym}$, and we write $\net^\mathrm{sym} = \net \cap \netex^\mathrm{sym}$. Thus, $\bx \in \netex^\mathrm{sym}$ implies that $\bx_\epsilon \in \catdigraph^\mathrm{sym}$ for all $\epsilon$.

\begin{definition}
An \textit{extended ultranetwork} is an extended network $\bx = (X, u_X)$ that satisfies the \textit{strong triangle inequality}: $u_X(x, x'') \leq \max\{u_X(x, x'), u_X(x', x'')\}$ for all $x, x', x'' \in X$. This implies that $\bx_\epsilon \in \catdigraph^\mathrm{trans}$ for all $\epsilon$. Denote by $\ultranetex$ the subcategory of $\netex$ whose objects are the extended ultranetworks, and the obvious analogues adding the modifiers ``symmetric'' or ``extended'': $\ultranet$, $\ultranetex^\mathrm{sym}$ and $\ultranet^\mathrm{sym}$.  

Any $\bx = (X, u_X) \in \ultranetex^\mathrm{sym}$ can be mapped to a function $T_X\colon \mathbb{R} \to \subpart(X)$, where $\subpart(X)$ is the set of partitions of subsets of $X$ and $T_X(\epsilon)$ is the quotient of $V(\bx_\epsilon)$ by the following equivalence relation: $x {\sim}_\epsilon x' \Leftrightarrow u_X(x, x') \leq \epsilon$. The map $T_X$, called a \textit{treegram}, satisfies the following properties:
\begin{itemize}
\item (Boundary conditions) $T_X(\epsilon) = \emptyset$, for $\epsilon < \epsilon_0 = \min \limits_{x, x' \in X} u_X(x, x')$; and $T_X(\epsilon) = T_X(\epsilon_1)$ for $\epsilon \geq \epsilon_1 = \max \{u_X(X \times X) \setminus \{+ \infty\} \}$.
\item (Hierarchy) For any $\epsilon \leq \epsilon'$, $x {\sim}_\epsilon x'$ implies $x \sim_{\epsilon'} x'$.
\item (Right continuity) For any $\epsilon$, there exists $\delta > 0$ such that $T_X(\epsilon) = T_X(\epsilon + r)$, for any $r \in [0, \delta]$.
\end{itemize}

The term ``treegram'' is taken from~\cite{MemoliHiReOf}, but here it is a bit more general since we don't need condition 2 in~\cite[Definition 1]{MemoliHiReOf}: that $T_X(\epsilon) = \{X\}$ for some $\epsilon$.  We depict a treegram as in Figure~\ref{fig: treegram}. Reading it from left to right, a vertex $x$ appears in the tree at the parameter $u_X(x,x)$ and merges with $x'$ at parameter $u_X(x, x')$. See Example~\ref{ex: treegram}.
\end{definition}

\begin{example} \label{ex: treegram}
Let $\bx = (X, u_X) \in \ultranetex^\mathrm{sym}$ be as follows: $X$ is the set $\{x_1, \ldots, x_4\}$ and $u$ is given by the following matrix:
\[
u = 
\begin{pmatrix}
0 & 3 & 4 & + \infty \\ 
3 & 1 & 4 & + \infty \\ 
4 & 4 & 2 & + \infty \\ 
+ \infty & + \infty & + \infty & 0
\end{pmatrix},
\]
where the entry $(i, j)$ of the matrix is $u_X(x_i, x_j)$. We can depict $\bx$ as in Figure~\ref{fig: treegram}.

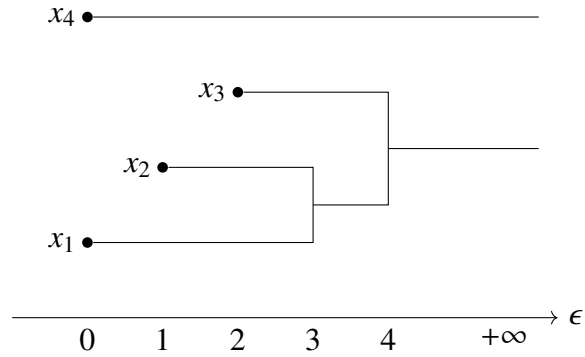
\begin{figure}[ht] 

\centering

\begin{tikzpicture}[sloped] 
\node[fill=black, circle, minimum size=4pt, inner sep=0pt] (a) at (0,1) {}; 
\node[fill=black, circle, minimum size=4pt, inner sep=0pt] (b) at (1,2) {}; 
\node[fill=black, circle, minimum size=4pt, inner sep=0pt] (c) at (2,3) {};
\node[fill=black, circle, minimum size=4pt, inner sep=0pt] (d) at (0,4) {};
\node (ab) at (3,1.5){};
\node (abc) at (4, 2.25){};
\draw  (a) -| (ab.center);
\draw  (b) -| (ab.center);
\draw  (c) -| (abc.center);
\draw  (ab.center) -| (abc.center);
\draw[->] (-1,0) --  (6.25,0)node[right]{$\epsilon$};

\node[left] at (a){$x_1$};
\node[left] at (b){$x_2$};
\node[left] at (c){$x_3$};
\node[left] at (d){$x_4$};
\node[below] at (0,0) {0};
\node[below] at (1,0) {1};
\node[below] at (2,0) {2};
\node[below] at (3,0) {3};
\node[below] at (4,0) {4};
\node[below left] at (6,0) {$+\infty$};
\draw[black] (d) -- (6, 4);
\draw[black] (abc.center) -- (6, 2.25);
\end{tikzpicture}

\caption{A graphical representation of the treegram associated to $\bx$ (see Example~\ref{ex: treegram}). Notice that two branches may never merge.}
\label{fig: treegram}

\end{figure}
\end{example}

\begin{remark} \label{remark:: about ffhat}
Let $\ff_1, \ff_2\colon \catdigraph \to \catdigraph$ with $\ff_1 \leq \ff_2$, and $\bx = (X, w_X) \in \netex$. Let $\ff_1(\bx) = (X, w^\ff_1)$ and $\ff_2(\bx) = (X, w^\ff_2)$. For any $\epsilon \in \mathbb{R}$, we have $\ff_1(\bx_\epsilon) \hookrightarrow \ff_2(\bx_\epsilon)$. Then for any $x, x' \in X$, we have $w^\ff_2(x, x') \leq w^\ff_1(x, x')$.
\end{remark}

\begin{definition}[Axiom of value] \label{def: axiom of value}
We say that the endofunctor $\ff\colon \catdigraph \to \catdigraph$ satisfies the \textit{axiom of value (for graphs)}, or simply \textbf{A1}, if $\ff(L_2) = D_2$ and $\ff(K_2) = K_2$. This definition is analogous to the one in~\cite{Carlssonetal13AxiCons}.
\end{definition}

Notice that the condition $\ff(L_2) = D_2$ implies $\ff(D_2) = D_2$.

\begin{remark}
In~\cite{Carlssonetal13AxiCons}, the authors are concerned with functors $$\mathfrak{H}\colon \disnet \to \disnet,$$ where $\disnet$ is the subcategory of $\net$ whose objects are \textit{dissimilarity networks:} pairs $(X, w_X)$ where $w_X(x, x') \geq 0$ for all $x, x' \in X$, and $w_X(x, x') = 0$ $\Leftrightarrow$ $x = x'$. Such functors are required to satisfy that $\by \coloneqq \mathfrak{H}(\bx)$ be a symmetric ultranetwork for all $\bx \in \disnet$. The \textit{axiom of value (for networks)} of~\cite{Carlssonetal13AxiCons} states that if for some $\epsilon$ one has $\bx_\epsilon = D_2$ or $L_2$ then $\by_\epsilon = D_2$, and if $\bx_\epsilon = K_2$ then $\by_\epsilon = K_2$. This definition in turn inspired Definition~\ref{def: axiom of value} . It is true that if $\ff \colon \catdigraph \to \catdigraph$ satisfies the axiom of value (for graphs) then $\ffh$ satisfies the axiom of value (for networks). However, not all functors $\mathfrak{H}$ are equal to some $\ffh$, as shown by Example \ref{ex: h not equal to ffh} below.

\begin{example} \label{ex: h not equal to ffh}
Let $\mathfrak{H} \colon \disnet \to \disnet$ be the \textit{grafting functor} of~\cite[Proposition 3]{Carlssonetal13AxiCons}. It is defined, for a given $\beta > 0$, as follows: 
\[
\big(\mathfrak{H}(\bx)\big)_\epsilon \coloneqq \begin{cases}
\ff^\mathrm{nrec}(\bx_\epsilon) \cap \ff^\mathrm{nrec}(\bx_\beta), & \epsilon \leq \beta, \\
\ff^\mathrm{rec}(\bx_\epsilon), & \epsilon > \beta,
\end{cases}
\]
for any $\bx  \in \disnet$.

Take $\beta = 2$ and $\bx = (X, w_X)$ where $X = \{x_1, x_2, x_3\}$ and the weight map $w_X$ in  matrix form is (see Figure~\ref{fig: ex-no-rep})
\[
w_X \coloneqq \begin{pmatrix}
0 & 2 & 4 \\
4 & 0 & 2 \\
2 & 4 & 0
\end{pmatrix}.
\] 
Thus, $X_\epsilon = \emptyset$ for $\epsilon < 0$, $\bx_\beta \cong C_3$, and
\[
\bx_\epsilon \cong \begin{cases}
D_3, & 0 \leq \epsilon < 2, \\
C_3, & 2 \leq \epsilon < 4, \\
K_3, & \epsilon \geq 4.
\end{cases}
\]

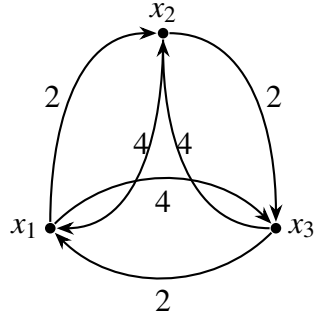
\begin{figure}[ht] 
\centering
\begin{tikzpicture}[>=Stealth,thick]
\node[fill=black, circle, minimum size=4pt, inner sep=0pt] (a1) at (0,0) {}; 
\node[fill=black, circle, minimum size=4pt, inner sep=0pt] (c1) at (3,0) {}; 
\node[fill=black, circle, minimum size=4pt, inner sep=0pt] (b1) at (1.5,1.5 * 1.73) {};
\draw[->] (a1) to [out=90,in=180]  node[anchor=east] {2} (b1);
\draw[->] (b1) to [out=-90,in=0] node[anchor=south] {4} (a1);
\draw[->] (b1) to [out=0,in=90] node[anchor=west] {2} (c1);
\draw[->] (c1) to [out=180,in=-90] node[anchor=south] {4} (b1);
\draw[->] (a1) to [out=45,in=135] node[anchor=north] {4} (c1);
\draw[->] (c1) to [out=-135,in=-45] node[anchor=north] {2} (a1);

\node[left] at (a1){$x_1$};
\node[above] at (b1){$x_2$};
\node[right] at (c1){$x_3$};
\end{tikzpicture} 
\caption{Dissimilarity network $\mathbf{X}$ from Example~\ref{ex: h not equal to ffh}} \label{fig: ex-no-rep}
\end{figure}

If $\mathfrak{H}$ were equal to $\ffh$, for some endofunctor $\ff$, we should have, for $\epsilon = 2$, 
\[
\ff(C_3) \cong \ff(\bx_\epsilon) = \big(\mathfrak{H}(\bx)\big)_\epsilon =  \ff^\mathrm{nrec}(\bx_\epsilon) \cap \ff^\mathrm{rec}(\bx_\beta) \cong K_3 \cap D_3 = D_3.
\]

Now considering $\by = (Y, w_Y)$ where $Y = X$ and $w_Y$ is obtained dividing $w_X$ by 2, we obtain, for $\epsilon = 1$,
\[
\ff(C_3) \cong \ff(\by_\epsilon) = \big(\mathfrak{H}(\by)\big)_\epsilon = \ff^\mathrm{nrec}(\by_\epsilon) \cap \ff^\mathrm{rec}(\by_\beta) \cong K_3 \cap K_3 = K_3.
\] 
Thus, $\ff(C_3) = D_3$ and $\ff_(C_3) = K_3$, a contradiction.
\end{example}

Example \ref{ex: h not equal to ffh} above shows that not all functors $\mathfrak{H}\colon \disnet \to \disnet$ arise as some $\ffh$. On the other hand, the functors $\ffh$ have as domain the category of extended networks, which contains the category of dissimilarity networks. Most functors $\mathfrak{H}\colon \disnet \to \disnet$ in~\cite{Carlssonetal13AxiCons} are equal to some $\ffh$: the reciprocal, non-reciprocal, semi-reciprocal, directed single linkage and unilateral functors. Another advantage of considering functors $\ffh$ is that they are simpler in the sense that we can work at the graph level instead of the network level: in order to determine $\ffh(\bx)_\epsilon$ all we need to consider is $\bx_\epsilon$, in contrast to the representable functors of~\cite{Carlsson10ClCl}, which implicitly assume knowledge of the network at (some) parameters greater than $\epsilon$. As a consequence, the proofs are much shorter, including the proof of stability (see Theorem~\ref{theo: stability}), which has as a corollary the stability of all functors mentioned in the previous sentence.
\end{remark}

\begin{remark} Property \textbf{A1} is not related to symmetry nor transitivity, as the following examples show:
\begin{itemize} 
\item $\ff$ transitive $\nRightarrow$ $\ff$ satisfies \textbf{A1}. To see this, consider $\ff^\mathrm{tc}$.
\item $\ff$ satisfies \textbf{A1} $\nRightarrow$ $\ff$ is transitive. Consider $\ff^\mathrm{ls}$. 
\item $\ff$ symmetric $\nRightarrow$ $\ff$ satisfies \textbf{A1}. Consider $\ff^\mathrm{us}$.
\item $\ff$ satisfies \textbf{A1} $\nRightarrow$ $\ff$ is symmetric. Consider $\ffoa$ where $\oa = \{(C_3, a_2, a_3)\}$. Then $\ffoa(C_3) = \ff^\mathrm{rev}(C_3) \notin \catdigraph^\mathrm{sym}$.
\end{itemize}
\end{remark}

\begin{definition} \label{def:: quotient graph}
Let $G = (V, E) \in \catdigraph$ and $P$ a partition of $V$. Consider the map $\pi\colon V \to P$ which assigns to each $x$ the unique block $B_x \in P$ containing $x$. Define $G_P \coloneqq (P, E_P)$, where $B \xrightarrow{G_P} B'$ if there are $x \in B$ and $x' \in B'$ such that $x \xrightarrow{G} x'$. 

Now given $\ff\colon \catdigraph \to \catdigraph^\mathrm{clust}$, consider the set $P_\ff(G)$ of connected components of $\ff(G)$. This set is a partition of $V$ which can be given as the following equivalence relation: $x \sim y$ $\Leftrightarrow$ $x \seta[\ff(G)] y$. Define $G_\ff \coloneqq G_P$, where $P = P_\ff(G)$.
\end{definition}

Notice that $G_P$ is the graph with vertex set $P$ that has the least possible number of arrows such that $\pi\colon G \to G_P$ is a graph map; that is: if $G'$ is another such graph, then $G_P \hookrightarrow G'$. 

\begin{proposition}
Let $\ffo$ be a representable functor with $D_1 \notin \Omega$. Then $\ffo$ satisfies \textbf{A1} $\Leftrightarrow$ for any $\omega \in \Omega$ and any partition $P = \{A, B\}$ of $V_\omega$ (the vertex set of $\omega$) into two blocks, we have $\omega_P = K_2$ (as in Definition~\ref{def:: quotient graph}).
\end{proposition}
\begin{proof}
The claim follows from the following equivalences:
\begin{itemize}
\item $\ffo(D_2) = D_2 \Leftrightarrow $ every $\omega \in \Omega$ is connected $\Leftrightarrow$ for any partition $P = \{A, B\}$ of $V_\omega$ into two blocks, we have $\omega_P \neq D_2$.

\item $\ffo(L_2) = D_2 \Leftrightarrow$ for any $\omega \in \Omega$, there is no surjective graph map $\omega \to L_2$ $\Leftrightarrow$ for any $\omega \in \Omega$ and any partition $P = \{A, B\}$ of $V_\omega$ into two blocks, we have $\omega_P = K_2$.

\item $\ffo(K_2) = K_2 \Leftrightarrow$ some $\omega \in \Omega$ has more than one vertex. \qedhere
\end{itemize}
\end{proof}

\begin{remark}
Theorem~\ref{theo: ffrec menor ff menor ffnrec} below states that $\ff^\mathrm{rec} \leq \ff \leq \ff^\mathrm{nrec}$ for all clustering functors $\ff$ satisfying \textbf{A1}. In order to prove this, we will adapt the ideas that led to the proof of~\cite[Theorem 4]{Carlssonetal13AxiCons} to the setting of functors $\ffh$ arising from functors $\ff\colon \catdigraph \to \catdigraph^\mathrm{clust}$.
\end{remark}

\begin{definition}[Extended axiom of value]
We say that $\ff\colon \catdigraph \to \catdigraph^\mathrm{clust}$ satisfies the extended axiom of value (or simply \textbf{A1}\textprime) if $\ff(T_n) = D_n$ and $\ff(K_n) = K_n$, $\forall n \in \N$.
\end{definition}

\begin{theorem} \label{theorem:: a1 implies a1'}
Let $\ff\colon \catdigraph \to \catdigraph^\mathrm{clust}$. Then, $\ff$ satisfies \textbf{A1} $\Leftrightarrow \ff$ satisfies \textbf{A1}{\textprime}.
\end{theorem}
\begin{proof}
Let's prove the ($\Rightarrow$) implication. Pick $a_i, a_j \in T_n$ with $i < j$. Consider the partition $A = \{a_k \; | \; k < j\}$, $B = \{a_k \; | \; k \geq j\}$ and the graph map $\phi\colon T_n \to L_2$ given by $\phi(A) = b_1$ and $\phi(B) = b_2$, where $b_1 \to b_2$ is the only arrow in $L_2$. By functoriality of $\ff$, $\phi\colon \ff(T_n) \to \ff(L_2) = D_2$ is a graph map. Thus, we cannot have $a_i \xrightarrow{\ff(T_n)} a_j$. Since $\ff(T_n)$ is symmetric and $a_i$ and $a_j$ were arbitrary, we have $\ff(T_n) = D_n$. That $\ff(K_n) = K_n$ follows from Proposition~\ref{coro:image of complete is complete}, since $\ff$ satisfies \textbf{A1}, which implies $\ff \neq \ff^\mathrm{disc}$.

Now let's prove the ($\Leftarrow$) implication. Taking $n=2$ gives us $T_2 = L_2$. This implies that $\ff(L_2) = D_2$ and $\ff(K_2) = K_2$.
\end{proof}

\begin{definition}
We say that a graph $G$ \textit{has no cycles} if there exists no sequence of points $v_1, \dots, v_m \in V$ with $v_1 \xrightarrow{G} v_2 \xrightarrow{G} \cdots \xrightarrow{G} v_m \xrightarrow{G} v_1$. This is equivalent to imposing that $\ff^{\{C_n\}}(G) = D_{(G)}$ for any $n \geq 2$, where $C_n$ is the cycle graph.
\end{definition}

\begin{lemma} \label{lemma: no cycles then etc}
If $G = (V, E) \in \catdigraph$ has no cycles, then there is a bijective map $\phi \colon V \to A_{T_n}$, where $n = |V|$ and $A_{T_n} = \{a_1, \dots, a_n\}$ is the vertex set of the graph $T_n$. Moreover, $\phi$ is a graph map.
\end{lemma}
\begin{proof}
For each $v \in G$, let $P(v) = \{v' \; | \; v' \xrightarrow{G} v\}$.

\begin{claim}
There exists $v^\ast_1 \in G$ such that $P(v^\ast_1) = \emptyset$.
\end{claim} 
\begin{claimproof}
If the claim were false, given $v_0 \in G$, there would exist $v_1, \ldots, v_n \in G$ such that $v_0 \xleftarrow{G} v_1 \xleftarrow{G} \cdots \xleftarrow{G} v_n$. Since $|V| = n$, we must have $v_i = v_j$ for some $i, j$, which would contradict the assumption that $G$ has no cycles.
\end{claimproof}

Now let $G_1 = G \setminus \{v^\ast_1\}$. Let $P_1(v) = \{v' \; | \; v' \xrightarrow{G_1} v\}$. By the above argument, there is $v^\ast_2 \in G_1$ such that $P_1(v^\ast_2) = \emptyset$, which implies $P(v^\ast_2) \subseteq \{v^\ast_1\}$. 

Let $G_2 = G_1 \setminus \{v^\ast_2\} = G \setminus \{v^\ast_1, v^\ast_2\}$ and let $P_2(v) = \{v' \; | \; v' \xrightarrow{G_2} v\}$. We then find $v^\ast_3 \in G_2$ such that $P_2(v^\ast_3) = \emptyset$, which implies $P(v^\ast_3) \subseteq \{v^\ast_1, v^\ast_2\}$. 

By the same argument, for $v^\ast_3, \ldots, v^\ast_n$ we will obtain $P(v^\ast_i) \subseteq \{v^\ast_1, \ldots, v^\ast_{i-1}\}$. Then, for any $i < j$, we have $v^\ast_j \nrightarrow v^\ast_i$ in $G$. Thus, the map $\phi\colon G \to T_n$ given by $\phi(v^\ast_i) = a_i, \; i = 1, \ldots, n$, is a bijective graph map.
\end{proof}

\begin{theorem} \label{theo: has no cycle then bijective map Tn}
Suppose $G = (V, E) \in \catdigraph$ has no cycles. If $\ff\colon \catdigraph \to \catdigraph^\mathrm{clust}$ satisfies \textbf{A1}, then $\ff(G) = D_{(G)}$.
\end{theorem}
\begin{proof}
Let $\phi \colon G \to T_n$ be the graph map of Lemma~\ref{lemma: no cycles then etc}, where $n = |V|$. By functoriality of $\ff$, $\phi\colon \ff(G) \to \ff(T_n)$ is a graph map. By Theorem~\ref{theorem:: a1 implies a1'}, $\ff(T_n) = D_n$. Thus, $\ff(G) = D_{(G)}$.
\end{proof}

\begin{theorem} \label{theo: ffrec menor ff menor ffnrec}
If $\ff\colon \catdigraph \to \catdigraph^\mathrm{clust}$ satisfies \textbf{A1}, then
\[
\ff^\mathrm{rec} \leq \ff \leq \ff^\mathrm{nrec}.
\]
\end{theorem}
\begin{proof}
Let $G = (V, E) \in \catdigraph$.

For the second inequality, let $\tilde{G} = G_{\ff^\mathrm{nrec}} = (\tilde{V}, \tilde{E})$ as in Definition~\ref{def:: quotient graph}.

\begin{claim}
$\tilde{G}$ has no cycles.
\end{claim}
\begin{claimproof}
Indeed, if we had a cycle $x_1 \to x_2 \to \dots \to x_n \to x_1$ in $\tilde{G}$, this would imply that there are $v_i, \tilde{v}_i \in x_i$ with 
\[
v_1 \leadsto \tilde{v}_1 \to v_2 \leadsto \tilde{v}_2 \to \dots \to v_n \leadsto \tilde{v}_n \to v_1 \textrm{ in $G$},
\]
hence $x_1  = x_2 = \dots = x_n$, a contradiction.
\end{claimproof}

Consider the graph map given by the projection $\pi\colon G \to \tilde{G}$, $\pi(v) = [v]$. Notice that $[v] \neq [v']$ is equivalent to $v \nrightarrow v'$ in $\ff^\mathrm{nrec}(G)$. Applying $\ff$, we obtain $\pi\colon \ff(G) \to \ff(\tilde{G}) = D_{(\tilde{G})}$, by Theorem~\ref{theo: has no cycle then bijective map Tn}. This implies that if $v, v' \in G$ satisfy $[v] \neq [v']$, then $v \nrightarrow v'$ in $\ff(G)$, or, by the contrapositive, if $v \xrightarrow{\ff(G)} v'$ then $v \xrightarrow{\ff^\mathrm{nrec}(G)} v'$, that is, $\ff(G) \leq \ff^\mathrm{nrec}(G) $.

Now for the leftmost inequality: if $v \xrightarrow{\ff^\mathrm{rec}(G)} v'$ then there are $v = v_1, v_2, \ldots, v_n = v'$ such that $v_1 \leftrightarrow v_2 \leftrightarrow \dots \leftrightarrow v_n$ in $G$. Consider the graph maps $\phi_i\colon K_2 \to (G, v_i, v_{i+1})$ for $i = 1, \ldots, n$. Then we also have graph maps $\phi_i\colon \ff(K_2) \to (\ff(G), v_i, v_{i+1})$. Since $\ff$ satisfies \textbf{A1}, $\ff(K_2) = K_2$, which implies $v_i \xrightarrow{\ff(G)} v_{i+1}$. By the transitivity of $\ff(G)$, we obtain $v \xrightarrow{\ff(G)} v'$. Thus, $\ff^\mathrm{rec} \leq \ff$.
\end{proof}

\begin{corollary}
Let $\ff\colon \catdigraph \to \catdigraph^\mathrm{clust}$ be a functor satisfying \textbf{A1}. For any extended network $\bx = (X, w_X)$, denoting $\ffh(\bx) = (X, u^\ff_X)$, $\ffh^\mathrm{nrec}(\bx) = (X, u_X^\mathrm{nrec})$ and $\ffh^\mathrm{rec}(\bx) = (X, u_X^\mathrm{rec})$, we have 
\[
u_X^\mathrm{nrec} \leq u^\ff_X \leq u_X^\mathrm{rec}.
\]
\end{corollary}
\begin{proof}
Follows from Theorem~\ref{theo: ffrec menor ff menor ffnrec} and Remark~\ref{remark:: about ffhat}.
\end{proof}

\begin{remark}
We can define, as in~\cite{Carlssonetal13AxiCons}, an alternative axiom of value: we say that $\ff\colon \catdigraph \to \catdigraph^\mathrm{trans}$ satisfies axiom \textbf{A1}{\textprime \textprime} if $\ff(D_2) = D_2$ and $\ff(L_2) = L_2$ (which implies $\ff(K_2) = K_2)$. By Proposition~\ref{prop:tc}, $\ff^\mathrm{tc}$ is the unique endofunctor on $\catdigraph$ that satisfies this axiom. 

Restricted to $\ultranet^\mathrm{sym}$, the map $\ffh^\mathrm{tc}$ is called \textit{network single linkage hierarchical clustering functor} (see~\cite{MemoliHiReOf}). 

When $\bx = (X, w_X)$ is a dissimilarity network and we write $\ff^\mathrm{tc}(\bx) = (X, u^\mathrm{tc})$, the value $u^\mathrm{tc}(x, x')$ is called the \textit{directed minimum chain cost of $(x, x')$} in~\cite{Carlssonetal13AxiCons}, and $\ffh^\mathrm{tc}$ is the \textit{directed single linkage}. Our result is in agreement with~\cite[Theorem 7]{Carlssonetal13AxiCons}, which states that directed single linkage is the unique functor satisfying \textbf{A1}{\textprime \textprime} (for dissimilarity networks).
\end{remark}

\begin{remark}
By Theorem~\ref{theo: f sym to sym}, $\ff(\catdigraph^\mathrm{sym}) \subset \catdigraph^\mathrm{sym}$, so, given any two functors $\ff_1\colon \catdigraph \to \catdigraph^\mathrm{sym}$ and $\ff_2\colon \catdigraph \to \catdigraph^\mathrm{trans}$, their composite $\ff_2 \circ \ff_1$ is a clustering functor. Examples are: the reciprocal clustering $\ff^\mathrm{rec} = \ff^\mathrm{tc} \circ \ff^\mathrm{ls}$; the non-reciprocal clustering $\ff^\mathrm{nrec} = \ff^\mathrm{tc} \circ \ffo$, where $\Omega = \{C_n\}_{n \in \mathbb{N}}$; the semi-reciprocal with size $t$, $\ff^\textrm{semi-$t$} = \ff^\mathrm{tc} \circ \ff^\mathrm{ls} \circ \ff^{[t]}$.

This gives us a simple recipe to construct clustering functors on (extended) networks: just take $\ffh$ for $\ff = \ff_2 \circ \ff_1$ as above.
\end{remark}

\subsection{Stability}

In~\cite{Carlsson2016ExcisiveHC} the authors defined a distance on the collection of all  dissimilarity networks, and in~\cite{Memoli2017DiIs} this distance, denoted by $d_\net$, was extended to the collection  of all networks. Using this distance, we can then ask whether a map $\ffh: \net \to \net$, for some endofunctor $\ff$, is \textit{stable} with respect to $d_\net$. More precisely: given two networks $\bx$ and $\by$, is it true that $d_\net \left( \ffh(\bx), \ffh(\by) \right) \leq d_\net(\bx, \by)$?

\begin{definition}[{\cite{Memoli2017DiIs}}]
Let $\bx = (X, w_X), \by = (Y, w_Y)$ be two networks. A \textit{correspondence} between $X$ and $Y$ is a subset $R \subseteq X \times Y$ such that the projections $\pi_X\colon R \to X$ and $\pi_Y\colon R \to Y$ are surjective. Denote by $C(X, Y)$ the set of all correspondences between $X$ and $Y$. The \textit{distortion} of $R$ (with respect to $\bx$ and $\by$) is the quantity
\[
\dis(R; \bx, \by) \coloneqq \max_{(x, y), (x', y') \in R} \big|w_X(x, x') - w_Y(y, y') \big|.
\]

The \textit{network distance} between $\bx$ and $\by$ is defined as
\[
d_\net(\bx, \by) \coloneqq \dfrac{1}{2} \inf_{R \in C(X, Y)} \dis(R; \bx, \by).
\]

We refer the reader to~\cite{Memoli2017DiIs} to see why this is a generalization of the Gromov-Hausdorff distance and the proof that it is actually a pseudometric, among other properties. 
\end{definition}  

\begin{theorem}[Stability] \label{theo: stability}
Let $\ff\colon \catdigraph \to \catdigraph$ be any endofunctor such that $\ff \neq \ff^\mathrm{disc}$. Let $\bx = (X, w_X)$ and $\by = (Y, w_Y)$ be any networks. Then,
\[
d_\net \left( \ffh(\bx), \ffh(\by) \right) \leq d_\net(\bx, \by).
\]
\end{theorem}
\begin{proof}
Let $R \in C(X, Y)$ be such that $\eta \coloneqq \dis(R; \bx, \by) = 2 d_\net(\bx, \by)$. Write $\ffh(\bx) = (X, w_X^\ff)$ and $\ffh(\by) = (Y, w_Y^\ff)$. 

Let $(x, y), (x', y') \in R$. Let $\phi\colon X \to Y$ be given by $\phi(x) = y$, $\phi(x') = y'$ and $\phi(a) = b$, for any choice of $b \in Y$ such that $(a, b) \in R$, $a \neq x, x'$.

\begin{claim}
The map $\phi$ induces a graph map $\phi_\epsilon\colon \bx_\epsilon \to \by_{\epsilon + \eta}$ given by $\phi_\epsilon(a) = \phi(a)$  $\forall\epsilon$.
\end{claim}

\begin{claimproof}
Indeed, suppose $a \seta[\bx_\epsilon] a'$, that is,  $w_X(a, a') \leq \epsilon$. Since $(a, \phi(a)), (a', \phi(a')) \in R$, we have
\[
|w_X(a, a') - w_Y(\phi(a), \phi(a'))| \leq \eta \Rightarrow
w_Y(\phi(a), \phi(a')) \leq w_X(a, a') + \eta \leq \epsilon + \eta,
\]
i.e., $\phi(a) \seta[\by_{\epsilon + \eta}] \phi(a')$. 
\end{claimproof}

Now let $\epsilon \coloneqq w_X^\ff(x, x')$, and so $x \seta[\ff(\bx_\epsilon)] x'$. Observe that $\epsilon$ is finite since $\ff \neq \ff^\mathrm{disc}$ and for some $\delta \in \R$ we will have that $\bx_\delta$ and $\ff(\bx_\delta)$ are complete graphs.

Further, notice that the map $\phi_\epsilon$ satisfies $\phi_\epsilon\colon (\bx_\epsilon, x, x') \to (\by_{\epsilon + \eta}, y, y')$. The functoriality of $\ff$ implies that $\phi_\epsilon\colon (\ff(\bx_\epsilon), x, x') \to (\ff(\by_{\epsilon + \eta}), y, y')$ is a graph map. Then $y \seta[\ff(\by_{\epsilon + \eta})] y'$, i.e., $w_Y^\ff(y, y') \leq \epsilon + \eta$. Thus,
\[
w_Y^\ff(y, y') - w_X^\ff(x, x') \leq \eta.
\]

Let $\psi \colon Y \to X$ be such that $\psi(y) = x$, $\psi(y') = x'$ and $(\psi(b), b) \in R$, for all $b \in Y$, and obtain
\[
w_X^\ff(x, x') - w_Y^\ff(y, y') \leq \eta.
\]

Both inequalities above lead to
\[
\left| w_X^\ff(x, x') - w_Y^\ff(y, y') \right| \leq \eta.
\]

But, since $(x, y), (x', y') \in R$ were chosen arbitrarily, we have
\[
\dis \left(R; \ffh(\bx), \ffh(\by) \right) \leq \eta = \dis(R; \bx, \by),
\]
which implies $d_\net \left( \ffh(\bx), \ffh(\by) \right) \leq d_\net(\bx, \by).$
\end{proof}

\begin{corollary}
The reciprocal $\mathfrak{H}^\mathrm{rec}$, non-reciprocal $\mathfrak{H}^\mathrm{nrec}$, unilateral $\mathfrak{H}^\mathrm{u}$ and directed single linkage $\mathfrak{H}^\mathrm{dsl}$ functors from~\cite{Carlssonetal13AxiCons} are stable.
\end{corollary}
\begin{proof}
Indeed, $\mathfrak{H}^\mathrm{nrec} = \ffh^\mathrm{nrec}$, $\mathfrak{H}^\mathrm{rec} = \ffh^\mathrm{rec}$, $\mathfrak{H}^\mathrm{u} = \ffh$ where $\ff = \ff^\mathrm{tc} \circ \ff^\mathrm{us}$, and $\mathfrak{H}^\mathrm{dsl} = \ffh^\mathrm{tc}$.
\end{proof}

\subsection{Hierarchical clustering of a graph}

An extended network can be seen as a sequence of nested graphs. The method studied in Section~\ref{subsec: hierarchical clustering} arose by applying an endofunctor to this sequence.

An alternative approach is to fix one graph and consider many ``nested'' endofunctors, as in the following definition.

\begin{definition}
Let $\Omega_1, \Omega_2, \dots, \Omega_n$ be families of graphs with $\Omega_1 \preccurlyeq \Omega_2 \preccurlyeq \ldots \preccurlyeq \Omega_n$. Let $\ff_i \coloneqq \ff^\mathrm{tc} \circ \ff^{\Omega_i}$, for all $i = 1, \dots, n$. Since $\ff^\mathrm{tc}$ is arrow increasing and $ \ff^{\Omega_i} \leq \ff^{\Omega_{i+1}}$, $i=1, \ldots, n-1$, we have $n$ clustering functors $\ff_1 \leq \ff_2 \leq \dots \leq \ff_n$. Take $\ff_0 \coloneqq \ff^\mathrm{disc}$ and $\ff_{n+1} \coloneqq \ff^\mathrm{comp}$. For a given graph $G = (V, E)$, define $u \colon V \times V \to \mathbb{R}$ by 
\[
u(v, v') \coloneqq \min \left\{i \; | \; v \xrightarrow{\ff_i(G)} v'\right\}.
\] 
We call $(V, u)$ the \textit{(symmetric extended) ultranetwork of $G$ obtained from $\Omega_1, \ldots, \Omega_n$}.
\end{definition}

Note that it is not interesting to have $\ff_i = \ff_j$ for $i \neq j$. The sets $\Omega_i = \{K_{i+1}\}$ are an example of this, since $\ff_i = \ff^\mathrm{tc} \circ \ff^\mathrm{us}$, for any $i \geq 1$. 

\begin{example} \label{exem: dendro grafo}
For $k = 1, \dots, 4$, let $\Omega_k = \{C_{k+1}\}$.
Let $G$ be the graph of Figure~\ref{fig: composition} (left). The ultranetwork of $G$ obtained from $\Omega_1, \ldots, \Omega_4$ is depicted in Figure~\ref{fig: composition} (right).
\end{example}

\begin{figure}[ht] 

\centering
\hspace*{0cm}\raisebox{2cm}{
\begin{tikzpicture}[>=Stealth,thick]
\node[fill=black, circle, minimum size=4pt, inner sep=0pt] (a) at (0,0) {}; 
\node[fill=black, circle, minimum size=4pt, inner sep=0pt] (b) at (3,0) {}; 
\node[fill=black, circle, minimum size=4pt, inner sep=0pt] (c) at (1.5,1.7) {};
\node[fill=black, circle, minimum size=4pt, inner sep=0pt] (d) at (4,1.7) {};
\draw[->] (a) -- (b);
\draw[->] (b) -- (a);
\draw[->]  (b) -- (c);
\draw[->]  (c) -- (a);
\draw[->]  (c) -- (d);
\node[left] at (a){$a$};
\node[right] at (b){$b$};
\node[above] at (c){$c$};
\node[right] at (d){$d$};
\end{tikzpicture}}
\
\begin{tikzpicture}[sloped]
\node[fill=black, circle, minimum size=4pt, inner sep=0pt] (a) at (0,1) {}; 
\node[fill=black, circle, minimum size=4pt, inner sep=0pt] (b) at (0,2) {}; 
\node[fill=black, circle, minimum size=4pt, inner sep=0pt] (c) at (0,3) {};
\node[fill=black, circle, minimum size=4pt, inner sep=0pt] (d) at (0,4) {};
\node (ab) at (1,1.5){};
\node (abc) at (2, 2.25){};
\node (abcd) at (4, 3.125){};
\draw  (a) -| (ab.center);
\draw  (b) -| (ab.center);
\draw  (c) -| (abc.center);
\draw  (d) -| (abcd.center);
\draw  (ab.center) -| (abc.center);
\draw  (abc.center) -| (abcd.center);
\draw[->] (-1,0) --  (6.25,0)node[right]{$\epsilon$};

\node[left] at (a){$a$};
\node[left] at (b){$b$};
\node[left] at (c){$c$};
\node[left] at (d){$d$};
\node[below] at (0,0) {0};
\node[below] at (1,0) {1};
\node[below] at (2,0) {2};
\node[below] at (3,0) {3};
\node[below] at (4,0) {4};
\node[below] at (5,0) {5}; 
\draw[black] (abcd.center) -- (5, 3.125);
\end{tikzpicture}

\caption{On the left: a graph $G$. On the right: the ultranetwork of $G$ obtained from $\Omega_1, \ldots, \Omega_4$. See Example~\ref{exem: dendro grafo}} 
\label{fig: composition}
\end{figure}
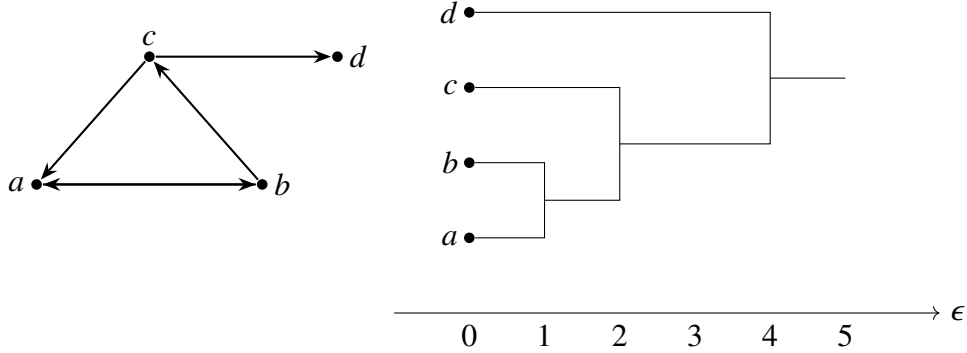

The following proposition states when $\Omega_i$ satisfies $\ff_i(L_2) = K_2$, for some $i$, then all $\ff_j$ with $j > i$ are redundant.

\begin{proposition}
If $\Omega$ is such that $\ffo(L_2) = K_2$, then $\ff^\mathrm{tc} \circ \ffo = \ff^\mathrm{conn}$ or $\ff^\mathrm{comp}$.
\end{proposition}
\begin{proof}
Let $G \in \catdigraph$. Then $\ffo(L_2) = K_2$ implies $\ff^\mathrm{us} \leq \ff$. Thus, 
\[
\ff^\mathrm{conn} = \ff^\mathrm{tc} \circ \ff^\mathrm{us} \leq \ff^\mathrm{tc} \circ \ff.
\]

By Proposition~\ref{prop: order in funct}, $\ff^\mathrm{tc} \circ \ff = \ff^\mathrm{conn}$ or $\ff^\mathrm{comp}$.
\end{proof}

\section{Discussion}
The definition and properties of endofunctors $\ff$ that carry a notion of density (as in~\cite{Carlsson10ClCl}) can be done with a more general construction relying on simplicial complexes. This will be the subject of an upcoming paper. 
See~\cite{hanbaek} for a rendition of the idea of motivic clustering in the context of graphons.

\subsection*{Acknowledgements} We thank Prof.\ Thiago de Melo for his  careful reading and detailed feedback of the manuscript and for his help with creating figures.

\bibliography{arxiv-bibliography}{}
\bibliographystyle{IEEEtran}

\end{document}